\documentclass[11pt]{article}

\usepackage[utf8]{inputenc} 
\usepackage[T1]{fontenc}    
\usepackage{lmodern}
\usepackage{hyperref}       
\usepackage{url}            
\usepackage{booktabs}       
\usepackage{amsfonts}       
\usepackage{nicefrac}       
\usepackage{microtype}      

\usepackage{amsmath,amsthm,amssymb,comment,bbm,algorithm,algorithmic,framed,hyperref,color,fullpage}
\usepackage{vfmacros-mod,framed,algorithm,algorithmic,bbm}

\usepackage[style=alphabetic,backend=bibtex,maxbibnames=20,maxcitenames=6,giveninits=true,doi=false,url=false]{biblatex}
\newcommand*{\citet}[1]{\AtNextCite{\AtEachCitekey{\defcounter{maxnames}{2}}} \textcite{#1}}
\newcommand*{\citetall}[1]{\AtNextCite{\AtEachCitekey{\defcounter{maxnames}{999}}} \textcite{#1}}
\newcommand*{\citep}[1]{\cite{#1}}

\newcommand{\INDSTATE}[1][1]{\STATE\hspace{#1\algorithmicindent}}

\newcommand{\ex}[2]{{\ifx&#1& \E \else \E_{#1} \fi \left[#2\right]}}
\newcommand{\var}[2]{{\ifx&#1& \Var \else \Var_{#1}\fi \left[#2\right]}}

\newcommand{\sdv}{\mathop{\mathrm{sd}}}

\newcommand{\dkl}[2]{\mathrm{D}\left(#1\middle\|#2\right)}
\newcommand{\prob}[2]{\pr_{#1}\left[#2\right]}
\newcommand{\nope}[1]{}

\newcommand\numberthis{\addtocounter{equation}{1}\tag{\theequation}}

\newtheorem{prop}[thm]{Proposition}

\newcommand{\interact}[2]{#1 {\rightarrow \atop \leftarrow} #2}

\iftrue
\newcommand{\vnote}[1]{\textcolor{red}{{\bf (Vitaly:} {#1}{\bf ) }} }
\newcommand{\tnote}[1]{\textcolor{blue}{{\bf (Thomas:} {#1}{\bf ) }}}
\else
\newcommand{\vnote}[1]{}
\newcommand{\tnote}[1]{}
\fi

\newcommand{\ind}{\mathbbm{1}}

\providecommand\X{\mathcal{X}}

\providecommand{\cP}{{\mathcal P}}
\providecommand{\cQ}{{\mathcal Q}}
\providecommand{\cY}{{\mathcal Y}}

\newcommand{\mymax}[2]{\max\left\{#1,#2\right\}}

\title{Calibrating Noise to Variance\\in Adaptive Data Analysis}

\author{Vitaly Feldman\thanks{Google Brain. Part of this work was done while at IBM Research -- Almaden and while visiting the Simons Institute, UC Berkeley  \dotfill  \texttt{vitaly@post.harvard.edu}} \and Thomas Steinke\thanks{IBM Research -- Almaden  \dotfill  \texttt{alkls@thomas-steinke.net}}}\date{}
\begin{document}

\maketitle

\begin{abstract}
Datasets are often used multiple times and each successive analysis may depend on the outcome of previous analyses. Standard techniques for ensuring generalization and statistical validity do not account for this adaptive dependence. A recent line of work studies the challenges that arise from such adaptive data reuse by considering the problem of answering a sequence of ``queries'' about the data distribution where each query may depend arbitrarily on answers to previous queries.

The strongest results obtained for this problem rely on differential privacy -- a strong notion of algorithmic stability with the important property that it ``composes'' well when data is reused. However the notion is rather strict, as it requires stability under replacement of an arbitrary data element. The simplest algorithm is to add Gaussian (or Laplace) noise to distort the empirical answers. However, analysing this technique using differential privacy yields suboptimal accuracy guarantees when the queries have low variance.

Here we propose a relaxed notion of stability based on KL divergence that also composes adaptively. We show that our notion of stability implies a bound on the mutual information between the dataset and the output of the algorithm and then derive new generalization guarantees implied by bounded mutual information. We demonstrate that a simple and natural algorithm based on adding noise scaled to the standard deviation of the query provides our notion of stability. This implies an algorithm that can answer statistical queries about the dataset with substantially improved accuracy guarantees for low-variance queries. The only previous approach that provides such accuracy guarantees is based on a more involved differentially private median-of-means algorithm and its analysis exploits stronger ``group'' stability of the algorithm.
\end{abstract}


\footnotetext{Accepted for presentation at Conference on Learning Theory (COLT) 2018.}

\section{Introduction}
The central challenge in statistical data analysis is to infer the properties of some unknown population given only a small number of samples from that population. While a plethora of techniques for guaranteeing statistical validity are available, few techniques can account for the effects of \emph{adaptivity}. Namely, if a single dataset is used multiple times, then the choice of which subsequent analyses to perform may depend on the outcomes of previous analyses. This adaptive dependence increases the risk of overfitting --- that is, inferring a conclusion that does not generalize to the underlying population.

To formalize this problem, \citetall{DworkFHPRR14:arxiv} and subsequent works \cite[][etc.]{HardtU14,SteinkeU15,BassilyNSSSU16,FeldmanS17} study the following question: How many data samples are necessary to accurately answer a sequence of queries about the data distribution when the queries are chosen adaptively -- that is, each query can depend on answers to previous queries? Each query corresponds to a procedure that the analyst wishes to execute on the data. The goal is to design an algorithm that provides answers to these adaptive queries that are close to answers that would have been obtained had each corresponding analysis been run on independent samples freshly drawn from the data distribution.

A common and relatively simple class of queries are statistical queries \citep{Kearns:98}. A statistical query is specified by a function $\psi:\X \rar [0,1]$ and corresponds to analyst wishing to compute the true mean $\ex{X\sim \cP}{\psi(X)}$ of $\psi$ on the data distribution $\cP$. (This is usually done by using the empirical mean $\frac{1}{n} \sum_{i=1}^n \psi(S_i)$ on a dataset $S$ consisting of $n$ i.i.d.~draws from the distribution $\cP$.)  For example, such queries can be used to to estimate the true loss (or error) of a predictor, the gradient of the loss function, or the moments of the data distribution. Standard concentration results imply that, given $n$ independent samples from $\cP$, $k$ fixed (i.e.~not adaptively-chosen) statistical queries can be answered with an additive error of at most $O \left( \sqrt{\log(k)/n} \right)$ with high probability by simply using the empirical mean of each query. At the same time it is not hard to show that, for a variety of simple adaptive sequences of queries, using the empirical mean to estimate the expectation leads to an error of $\Omega(\sqrt{k/n})$ \citep{DworkFHPRR14:arxiv}. Equivalently, in the adaptive setting, the number of samples required to ensure fixed error scales linearly (rather than logarithmically in the non-adaptive setting) with the number of queries and, in particular, in the worst case, using empirical estimates gives the same guarantees as using fresh samples for every query (by splitting the dataset into $k$ parts).

\citet{DworkFHPRR14:arxiv} showed that, remarkably, it is possible to quadratically improve the dependence on $k$ in the adaptive setting by simply perturbing the empirical answers. Specifically, let $S \in \X^n$ denote a dataset consisting of $n$ i.i.d.~samples from some (unknown) probability distribution $\cP$. Given $S$, the algorithm receives $k$ adaptively-chosen statistical queries $\psi_1, \ldots, \psi_k : \X \to [0,1]$ one-by-one and provides $k$ approximate answers $v_1, \ldots, v_k  \in \R$. Namely, $v_j = \frac{1}{n} \sum_{i = 1}^n \psi_j(S_i) + \xi_j$, where each ``noise'' variable $\xi_j$ is drawn independently from $\mathcal{N}(0,\sigma^2)$. The results of \citet{DworkFHPRR14:arxiv} and subsequent sharper analyses \citep{BassilyNSSSU16,Steinke16} show that
, with high probability (over the drawing of the sample $S \sim \cP^n$, the noise $\xi$, and the choice of queries), we have the following guarantee \begin{equation}\forall j \in \{1, \cdots, k\} ~~~~~ \left| v_j - \ex{X \sim \cP}{\psi_j(X)} \right| \leq O \left( \sqrt{\frac{\sqrt{k \log k}}{n}} \right).\label{eqn:additive}\end{equation}
This quadratic relationship between $n$ and $k$ was also shown to be optimal in the worst case \citep{HardtU14,SteinkeU15}.

The approach of \citet{DworkFHPRR14:arxiv} relies on properties of differential privacy \citep{DworkMNS:06,DworkKMMN06} and known differentially private algorithms. Differential privacy is a stability property of an algorithm, namely it requires that replacing any element in the input dataset results in a small change in the output distribution of the algorithm. Specifically, a randomized algorithm $M : \X^n \to \mathcal Y$ is $(\varepsilon,\delta)$-differentially private if, for all datasets $s,s' \in \X^n$ that differ on a single element and all events $E \subseteq \mathcal{Y}$, $$\prob{}{M(s) \in E} \leq e^\varepsilon \prob{}{M(s')\in E} + \delta.$$

This stability notion implies that a function output by a differentially private algorithm on a given dataset generalizes to the underlying distribution \citep{DworkFHPRR14:arxiv,BassilyNSSSU16}. Specifically, if a differentially private algorithm is run on a dataset drawn i.i.d~from any distribution and the algorithm outputs a function, then the empirical mean of that function on the input dataset is close to the expectation of that function on sample from the same distribution.

The second crucial property of differential privacy is that it composes adaptively: running several differentially private algorithms on the same dataset still is differentially private (with somewhat worse parameters) even if each algorithm depends on the output of all the previous algorithms. This property makes it possible to answer adaptively-chosen queries with differential privacy and a number of algorithms have been developed for answering different types of queries. The generalization property of differential privacy then implies that such algorithms can be used to provide answers to adaptively-chosen queries while ensuring generalization \citep{DworkFHPRR14:arxiv}. Specifically, the algorithm for answering statistical queries mentioned above is based on the most basic differentially private algorithm: perturbation by adding Laplace or Gaussian noise \citep{DworkMNS:06}.

Differential privacy requires that the output distribution of an algorithm does not change much when any element of a dataset is replaced with an arbitrary other element in the domain $\X$. As a result, the amount of noise that needs to be added to ensure differential privacy scales linearly with the range of the function $\psi$ whose expectation needs to be estimated. If the range of $\psi$ is comparable to the standard deviation of $\psi(x)$ on $x$ drawn from $\cP$ (such as when $\psi$ has range $\zo$ and mean $1/2$) then the error resulting from addition of noise is comparable to the standard deviation of $\psi$. However, for queries whose standard deviation is much lower than the range, the error introduced by noise is much worse than the sampling error. Variance is much smaller than the range for a variety of common settings, for example, difference between candidate predictors for the same problem or individual input features when the input is usually sparse.

Achieving error guarantees in the adaptive setting that scale with the standard deviation instead of range is a natural problem. Recently, \citet{FeldmanS17}
gave a different algorithm that achieves such a guarantee. Specifically, their algorithm ensures that with probability at least $1-\beta$, \begin{equation}\forall j \in \{1, \cdots, k\} ~~~~~ \left| v_j - \ex{X \sim \cP}{\psi_j(X)} \right| \leq \sdv(\psi_j(\cP)) \cdot O \left( \sqrt{\frac{\sqrt{k \log^3 (k/\beta)}}{n}} \right) + \beta,\label{eqn:sdv}\end{equation} where $\sdv(\psi_j(\cP)) = \sqrt{\ex{Y \sim \cP}{(\psi_j(Y)-\ex{X \sim \cP}{\psi_j(X)})^2}}$ is the standard deviation of $\psi_j$ on the distribution $\cP$ and $\beta>0$ can be chosen arbitrarily.
Their algorithm is based on an approximate version of the median of means algorithm and its analysis still relies on differential privacy. (Their results extend beyond statistical queries, but we restrict our attention to statistical queries in this paper.)

In this work, we ask: does the natural algorithm that perturbs the empirical answers with noise scaled to the standard deviation suffice to answer adaptive queries with accuracy scaling to sampling error? To answer this seemingly simple question, we address a more fundamental problem: does there exist a notion of stability that has the advantages of differential privacy (namely, allows adaptive composition and implies generalization) but avoids the poor dependence on the worst-case sensitivity of the query. This algorithm was analyzed by \citet{BassilyFreund:16} via a notion of typical stability they introduced. Their analysis shows that the algorithm will ensure the correct scaling of the error with standard deviation but it does not improve on the naive mechanisms in terms of scaling with $k$.
Several works have considered relaxations of differential privacy in this context. For example, \citetall{BassilyNSSSU16} considered a notion of stability based on using KL divergence or total variation distance in place of differential privacy (which can be defined in terms of approximate max divergence). \citetall{Wang16} considered the expected KL divergence between the output of the algorithm when run on a random i.i.d~dataset versus the same dataset with one element replaced by a fresh sample; unfortunately, their stability definition does \emph{not} compose adaptively. Notions based on the mutual information between the dataset and the output of the algorithm and their relationship to differential privacy have also been studied \citep{DworkFHPRR15:arxiv,RussoZ16,RogersRST16,RaginskyRTWX16,XuR17}. However, to the best of our knowledge, these approaches do not give a way to analyze the calibrated noise addition that ensures correct dependence on $k$.

\subsection{Our Contributions}
We introduce new stability-based and information-theoretic tools for analysis of the generalization of algorithms in the adaptive setting. The stability notion we introduce is easier to satisfy than differential privacy, yet has the properties crucial for application in adaptive data analysis. These tools allow us to demonstrate that calibrating the variance of the perturbation to the empirical variance of the query suffices to ensure generalization, as long as the noise rate does not become too small. To ensure this lower bound on the noise rate we simply add a second order term to the variance of the perturbation.  Specifically, our algorithm is described in Figure \ref{fig:scaledgauss}. The only difference between our algorithm and previous work \citep{DworkFHPRR14:arxiv,BassilyNSSSU16} is that in prior work the variance of the Gaussian perturbation is fixed.

\begin{figure}[h!]
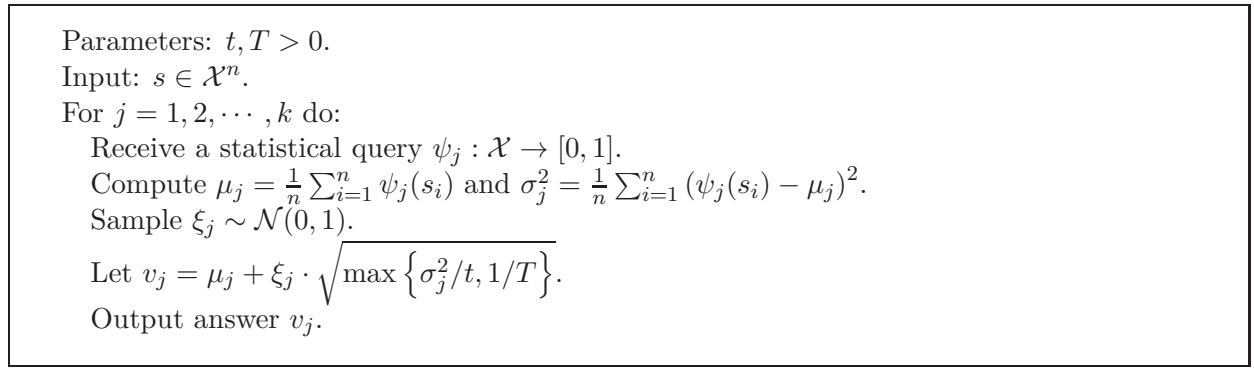

\begin{framed}
\begin{algorithmic}
\INDSTATE[0]{Parameters: $t,T>0$.}
\INDSTATE[0]{Input: $s \in \X^n$.}
\INDSTATE[0]{For $j=1,2, \cdots, k$ do:}
\INDSTATE[1]{Receive a statistical query $\psi_j : \X \to [0,1]$.}
\INDSTATE[1]{Compute $\mu_j = \frac{1}{n} \sum_{i=1}^n \psi_j(s_i)$ and $\sigma_j^2 = \frac{1}{n} \sum_{i = 1}^n \left( \psi_j(s_i) - \mu_j \right)^2$.}
\INDSTATE[1]{Sample $\xi_j \sim \mathcal{N}(0,1)$.}
\INDSTATE[1]{Let $v_j = \mu_j + \xi_j \cdot \sqrt{\mymax{\sigma_j^2/t}{1/T}}$.}
\INDSTATE[1]{Output answer $v_j$.}
\end{algorithmic}
\end{framed}
\caption{Calibrating noise to variance for answering adaptive queries.}\label{fig:scaledgauss}
\end{figure}

We prove that this algorithm has the following accuracy guarantee.

\begin{thm}[Main Theorem]\label{thm:main-intro}
Let $\cP$ be a distribution on $\X$ and let $M$ be our algorithm from Figure \ref{fig:scaledgauss} instantiated with $T=n^2/k$ and $t=n\sqrt{2\ln(2k)/k}$. Suppose $M$ is given a sample $S \sim \cP^n$ and is asked adaptive statistical queries $\psi_1, \cdots, \psi_k : \X \to [0,1]$. Then $M$ produces answers $v_1, \cdots, v_k \in \R$ satisfying the following. $$ \ex{}{\max_{j=1}^k \frac{|v_j - \ex{X \sim \cP}{\psi_j(X)}|}{\mymax{\tau \cdot \sdv(\psi_j(\cP))}{\tau^2}}} \leq 4,~~~~~\text{where}~~~~~ \tau = \sqrt{\frac{\sqrt{2k \ln (2k)}}{n}}.$$
\end{thm}
Intuitively (that is, ignoring the second term in the maximum), the conclusion of Theorem \ref{thm:main-intro} states that, with good probability, the error in each answer scales as the standard deviation of the query multiplied by $\tilde{O}\left(\sqrt{\sqrt{k}/n}\right)$ --- which is what would be expected if we used $n/\sqrt{k}$ fresh samples for each query. The $\ln k$ factor arises from the fact that we take a union bound over the $k$ queries.

More precisely, applying Markov's inequality to the conclusion of Theorem \ref{thm:main-intro}, shows that, with probability at least $90\%$, \begin{equation}\forall j ~~~~ \left| v_j - \ex{X \sim \cP}{\psi_j(X)} \right| \leq 40 \cdot \mymax{\tau \cdot \sdv(\psi_j(\cP))}{\tau^2} \leq \sdv(\psi_j(\cP)) \cdot 40\sqrt{\frac{\sqrt{2k \ln (2k)}}{n}} + 40 {\frac{\sqrt{2k \ln (2k)}}{n}}.\label{eqn:sdv2}\end{equation}
This guarantee is directly comparable to the earlier bound \eqref{eqn:sdv} of \citet{FeldmanS17} -- though it is weaker in two ways: First, Theorem \ref{thm:main-intro} is a bound on the expectation and does not readily yield high probability bounds (other than via Markov's inequality). Second, the second term in the maximum (which we think of as a low-order term) still depends linearly on the sensitivity and is potentially larger. The advantage of this algorithm is that it is substantially simpler than the earlier work.

\medskip

Now we turn to the analysis tools that we introduce. Clearly the \emph{empirical error} of our algorithm --- that is $|v_j - \mu_j|$ --- scales with the empirical standard deviation $\sigma_j$. However, we must bound the \emph{true error}, namely $|v_j - \ex{X \sim \cP}{\psi_j(X)}|$. By the triangle inequality, it suffices to bound the \emph{generalization error} $| \mu_j - \ex{X \sim \cP}{\psi_j(X)}|$ in terms of standard deviation and to relate the empirical standard deviation $\sigma_j$ to the true standard deviation $\sdv(\psi_j(\cP))$.

\subsubsection{Average leave-one-out KL stability and generalization}
The key to our analysis is the following stability notion.

\newcommand{\KLAS}{ALKL stable}
\begin{defn}[Average Leave-one-out KL stability]\label{defn:klas-intro}
An algorithm $M : \left(\X^n \cup \X^{n-1} \right) \to \mathcal{Y}$ is $\varepsilon$-\KLAS{} if, for all $s \in \X^n$, $$\frac{1}{n} \sum_{i \in [n]} \dkl{M(s)}{M(s_{-i})} \leq \varepsilon,$$ where $s_{-i} \in \X^{n-1}$ denotes $s$ with the $i^\text{th}$ element removed. Here $\dkl{\cdot}{\cdot}$ denotes the Kullback-Leibler divergence.
\end{defn}

Our notion differs from differential privacy in three significant ways.\footnote{These relaxations mean that ALKL stability is \emph{not} a good privacy definition, in contrast to differential privacy. In particular, because of the averaging, ALKL stability cannot distinguish between an algorithm that offers good privacy to all individuals and one that offers great privacy for $n-1$ individuals but terrible privacy for the last individual. Compromising a single data point is, however, not an issue for generalization.} First, we use stability to leaving one out (LOO) rather than replacing one element. Second, we average the stability parameter across the $n$ dataset elements. Third, we use KL divergence instead of (approximate) max divergence. This is necessary to obtain stronger bounds for our calibrated noise addition as our algorithm does not satisfy differential privacy with parameters that would be suitable to ensure generalization. We note that average LOO stability is a well-studied way to define algorithmic stability for the loss function (\eg, \citep{BousquettE02,PoggioRMN04}).
The use of KL divergence appears to be necessary to ensure adaptive composition of our averaged notion. Specifically, the following composition result is easy to prove.
\begin{lem}[Composition, see Lemma \ref{lem:composition}]\label{lem:composition-intro}
Suppose $M : \left(\X^n \cup \X^{n-1} \right) \to \mathcal{Y}$ is $\varepsilon$-\KLAS{} and $M' : \mathcal{Y} \times \left(\X^n \cup \X^{n-1} \right) \to \mathcal{Z}$ is such that $M'(y,\cdot) : \left(\X^n \cup \X^{n-1} \right) \to \mathcal{Z}$ is $\varepsilon'$-\KLAS{} for all $y \in \mathcal{Y}$. Then the composition $s \mapsto M'(M(s),s)$ is $(\varepsilon+\varepsilon')$-\KLAS{}.
\end{lem}

Using composition, we can show that our algorithm (Figure \ref{fig:scaledgauss}, with the parameters set as in Theorem \ref{thm:main-intro}) is $\frac{kt}{n^2}$-\KLAS{}. In particular, we show that each one of the $k$ answers is computed in a way that is $\frac{t}{n^2}$-\KLAS{}. This follows from the properties of the KL divergence between Gaussian distributions and the way we calibrate the noise. (Alternatively, we could use Laplace noise to obtain similar results.)

We note that $\sqrt{2\varepsilon}$-differential privacy \citep{DworkMNS:06}, notions based on Renyi differential privacy \citep{BunS16,Mironov17}, and $\varepsilon$-KL-stability \citep{BassilyNSSSU16} all imply $\varepsilon$-ALKL stability\footnote{It may be necessary to extend an algorithm satisfying one of these definitions to inputs of size $n-1$ to satisfy ALKL stability. This can be done by simply padding such an input with one arbitrary item.} Thus we can also compose any ALKL stable algorithm with any of the many algorithms satisfying one of the aforementioned definitions.

Crucially, average KL-divergence is strong enough to provide a generalization guarantee that scales with the standard deviation of the queries, as we require. Our proof is based on the high-level approach introduced by \citet{DworkFHPRR15:arxiv} who first convert a stability guarantee to an upper bound on information between the input dataset and the output of the algorithm and then derive generalization bounds from the bound on information. Here, we demonstrate that ALKL stability implies a bound on the mutual information between the input and output of the algorithm when run on independent samples and then derive generalization guarantees from the bound on mutual information \footnote{We thank Adam Smith for suggesting that we try this approach to proving generalization for ALKL stable algorithms.}
\begin{prop}[see Prop.~\ref{prop:mi}]\label{prop:gen-intro}
Let $M : \left( \mathcal{X}^n \cup \mathcal{X}^{n-1} \right) \to \mathcal{Q}$ be $\varepsilon$-\KLAS{}. Let $S \in \mathcal{X}^n$ consist of $n$ independent samples from some distribution $\cP$. Then \begin{equation}I(S;M(S)) \leq \varepsilon n,\label{eqn:MI-men}\end{equation} where $I$ denotes mutual information.
\end{prop}

To prove Proposition \ref{prop:gen-intro}, we introduce an intermediate notion of stability:
\newcommand{\MIS}{MI stable}
\begin{defn}[Mutual Information Stability]
\label{def:mis}
A randomized algorithm $M : \X^n \to \mathcal{Y}$ is $\varepsilon$-\MIS{} if, for any random variable $S$ distributed over $\X^n$ (including non-product distributions), $$\frac{1}{n} \sum_{i=1}^n I(M(S);S_i|S_{-i}) \leq \varepsilon.$$
\end{defn}
This notion is based on the notion of stability studied in \citep{RaginskyRTWX16} that considers only product distributions over the datasets and, as a result, does not compose adaptively.

We prove Proposition \ref{prop:gen-intro} by combining the following two facts.
\begin{itemize}
\item[(i)] $\varepsilon$-ALKL stability implies $\varepsilon$-MI stability. (Lemma \ref{lem:klas-mis}) To show this, we express $I(M(S);S_i|S_{-i})$ as the expectation over $S$ of the KL divergence of the distribution (over the randomness of $M$) of $M(S)$ from an appropriately weighted convex combination of distributions $M(S')$. (Specifically, $S'$ is $S$ with $S_i$ ``resampled.'') The ``mean-as-minimizer'' property of KL divergence (Lemma \ref{lem:klmin}) means we can simply replace this convex combination with $M(S_{-i})$ to complete the proof.
\item[(ii)] $\varepsilon$-MI stability implies the mutual information bound \eqref{eqn:MI-men}. (Lemma \ref{lem:mis-mi}) To prove this, we invoke the chain rule for mutual information along with the fact that $S_i$ is independent from $S_{-i}$ (which helps resolve the conditioning).
\end{itemize}
Further,  we point out that mutual information stability composes adaptively in the same way as ALKL stability and hence could be useful for understanding adaptive data analysis for more general queries (\eg unlike ALKL stability it does not require $M(S_{-i})$ to be defined).


As first shown in the context of PAC-Bayes bounds \citep{McAllester13} and more recently in \citep{RussoZ16}, a bound on mutual information implies generalization results. Using a similar technique, we show that, if the mutual information $I(S;\psi_j)$ is small (with $S$ consisting of $n$ i.i.d.~draws from $\cP$), we have $\ex{}{\frac{1}{n} \sum_{i=1}^n \psi_j(S_i)} \approx \ex{X \sim \cP}{\psi_j(X)}$. Moreover, the quality of the approximation scales with the standard deviation. (Specifically, the approximation bound depends on the moment generating function $\ex{}{e^{\lambda \mu_j}}$ of $\mu_j = \frac{1}{n} \sum_{i=1}^n \psi_j(S_i)$, which we bound using both the variance and the range of $\mu_j$.) We can similarly relate the empirical variance $\sigma_j^2$ to the true variance. Thus a bound on mutual information suffices to bound generalization error and, thus, prove Theorem \ref{thm:main-intro}.

Another known implication of bounded mutual information is that any event that would happen with sufficiently low probability on fresh data will still happen with low probability \citep{RussoZ16,RogersRST16}. In particular, if $E$ is some ``bad'' event -- such as overfitting the data or making a false discovery -- and we know that we are exponentially unlikely to overfit fresh data $S'$, then the probability of $M$ overfitting its input data $S$ is also small, provided the mutual information is small. (See Section \ref{sec:events} for additional details.)

One downside of using mutual information is that does not allow us to prove high probability bounds, as can be done with differential privacy and the notion of approximate max-information \citep{DworkFHPRR15:arxiv}. We note, however, that our analysis still upper bounds the expectation of the largest error among all the queries that were asked. In other words, a union bound over queries is built into the guarantees of the algorithm. Using known techniques, the confidence can be amplified at the expense of a somewhat more complicated algorithm. In addition, our algorithm yields stronger stability guarantees than just ALKL stability. For example, the minimum noise level of $1/T$ ensures differential privacy (albeit with relatively large parameters\footnote{Specifically, with the parameter setting from Theorem \ref{thm:main-intro}, our algorithm satisfies $\left(O\left(\sqrt{\log(1/\delta)}\right),\delta\right)$-differential privacy for all $\delta>k^{-\Omega(k)}$.}). The parameters can be improved using the averaging over the indices that we use in ALKL stability but that leads to a notion that does not appear to compose adaptively. Using a different analysis technique it might be possible to exploit the stronger stability properties of our algorithm to prove high probability generalization bounds. We leave this as an open problem.  On the other hand, stability with KL divergence is easier to analyze and allows a potentially wider range of algorithms to be used.

\subsection{Related work}
Our use of mutual information to derive generalization bounds is closely related to PAC-Bayes bounds first introduced by \citet{McAllester99} and extended in a number of subsequent works (see \citep{McAllester13} for an overview). In this line of work, the expected generalization error of a predictive model (such as classifier) randomly chosen from some data-dependent distribution $\cQ(S)$ is upper-bounded by the KL divergence between $\cQ$ and an arbitrary data-independent prior distribution $\cP_0$. One natural choice of $\cQ(S)$ is the output distribution of a randomized learning algorithm $\cA$ on $S$. By choosing the prior $\cP_0$ to be the distribution of the output of $\cA$ on a dataset drawn from $\cP^n$ one obtains that the expected generalization error is upper-bounded by the expected KL divergence between $\cQ(S)$ an $\cP_0$ \citep{McAllester13}. While this has not been pointed out in \citep{McAllester13}, this is exactly the mutual information between $S$ and $\cA(S)$.

Recently, interest in using information-based generalization bounds was revived by applications in adaptive data analysis \citep{DworkFHPRR15:arxiv}.
Specifically, \citet{DworkFHPRR15:arxiv} demonstrate that approximate max-information  between the input dataset and the output of the algorithm (a notion based on the infinity divergence between the joint distribution and the product of marginals) implies generalization bounds with high probability. They also showed that $(\eps,0)$-differential privacy implies an upper bound on approximate max-information (and this later extended to $(\eps,\delta)$-differential privacy by \citetall{RogersRST16}). \citet{RussoZ16} show that mutual information can also be used to derive bounds on expected generalization error and discuss several applications of these bounds. \citet{XuR17} show how to derive ``low-probability'' bounds on the generalization error in this context. (We note that \citep{RussoZ16,XuR17} use the same technique as that used in PAC-Bayes bounds and appear to have overlooked the direct connection between their results and the PAC-Bayes line of work.)

Recent work \citep{BassilyMNSY18} studies learning algorithms in the PAC model whose output has low mutual information with the input dataset. They also discuss generalization bounds based on mutual information and (independently) derive results similar to those we give in Section \ref{sec:events}.

\section{Notation, Definitions, \& Key Properties} 

We use $X \sim \cP$ to denote that $X$ is drawn from the distribution $\cP$. For the most part, we adopt the convention that upper-case letters denote random variables, whereas lower-case letters denote realizations thereof.
For $n \in \mathbb{N}$, we denote $[n]=\{1,2,\cdots,n\}$ and $S \sim \cP^n$ denotes that $S=(S_1, \cdots, S_n)$ consists of $n$ independent draws from the distribution $\cP$. We use $S_i$ to denote the $i^\text{th}$ element of $S$ and $S_{-i} = (S_1, \cdots, S_{i-1}, S_{i+1}, \cdots, S_n)$ to denote the other $n-1$ elements.
For two random variables $X$ and $Y$ and a realization $x$ of $X$, we use the notation $Y|X=x$ to denote the conditional distribution of $Y$ given $X=x$. 

For a distribution $\cP$ on $\X$ and a function $\psi : \X \to \R$, we use $\psi(\cP)$ to denote the distribution on $\R$ obtained by applying $\psi$ to a random sample from $\cP$. The mean of this distribution is denoted $\cP[\psi]=\ex{X \sim \cP}{\psi(X)}$. We use $\sdv(\psi(\cP))=\sqrt{\var{}{\psi(\cP)}} = \sqrt{\ex{X \sim \cP}{\psi(X)^2} - \ex{X \sim \cP}{\psi(X)}^2}$ to denote the standard deviation of this distribution. We also interpret a tuple $s \in \X^n$ as a distribution --- namely the distribution obtained by selecting $s_i$ for a random $i \in [n]$ --- and we analogously define the empirical mean and standard deviation: $s[\psi] = \frac{1}{n} \sum_{i=1}^n \psi[s_i]$ and $\sdv(\psi(s)) = \sqrt{\var{}{\psi(s)}} = \sqrt{\frac{1}{n} \sum_{i=1}^n (\psi(s_i) - S[\psi])^2}$.

\subsection{KL Divergence}

Before continuing, we first establish some relevant properties of the KL divergence. See the textbook by \citet{CoverT12} for an introduction to the properties of KL divergence (a.k.a.~relative entropy).

First we state the definition of KL divergence for completeness.

\begin{defn}
Let $\cP$ and $\cQ$ be probability distributions on a space $\Omega$. Suppose $\cP$ is absolutely continuous with respect to $\cQ$. Then the KL divergence from $\cQ$ to $\cP$ is $$\dkl{\cP}{\cQ} = \ex{X \sim \cP}{\ln \left(\frac{\cP(X)}{\cQ(X)}\right)},$$ where $\cP(x)$ and $\cQ(x)$ denote the probability mass or density functions of $\cP$ and $\cQ$ respectively evaluated at the point $X$. (More generally, $\cP(x)/\cQ(x)$ denotes the Radon-Nikodym derivative of $\cP$ with respect to $\cQ$ evaluated at $x$.)
\end{defn}

In some cases we will abuse notation and refer to $\dkl{X}{Y}$ where $X$ and $Y$ are ``random variables'' rather than formally-defined distributions. This should be read to be the divergence between the distribution of $X$ and the distribution of $Y$.

We state the well-known chain rule:
\begin{lem}[{\citep{CoverT12}, Theorem 2.5.3}]\label{lem:KLchain}
Let $\cP$ and $\cQ$ be two distributions over some domain $\X\times \cY$.
Then $$\dkl{\cP(x,y)}{\cQ(x,y)} = \dkl{\cP(x)}{\cQ(x)} + \ex{x' \sim \cP}{\dkl{\cP(y|x=x')}{\cQ(y|x=x')}}.$$
\end{lem}
Here $\cP(x)$ (or $\cQ(x)$) denotes the marginal distributions of $\cP$ (or $\cQ$) over $\X$ and $\cP(y|x=x')$ denotes the marginal distribution of $\cP$ on $\cY$ conditioned on $x=x'$.

We begin by looking at the KL divergence between two Gaussian distributions, as this is what our mechanism uses. Recall that the Gaussian (or normal) distribution with mean $\mu$ and variance $\sigma^2$ --- denoted $\mathcal{N}(\mu,\sigma^2)$ --- has a probability density at $x$ given by $\frac{1}{\sqrt{2\pi\sigma^2}} \exp\left(-\frac{(x-\mu)^2}{2\sigma^2}\right)$.
\begin{lem}[{\citep{GilAL13}, Table 3}]\label{lem:KLG}
Let $\mu, \tilde \mu, \sigma, \tilde \sigma \in \mathbb{R}$. Then $$\dkl{\mathcal N(\mu,\sigma^2)}{\mathcal N(\tilde\mu,\tilde\sigma^2)} = \frac{(\mu-\tilde\mu)^2}{2\tilde\sigma^2} + \frac12 \left( \frac{\sigma^2}{\tilde\sigma^2} -1 - \ln \left( \frac{\sigma^2}{\tilde\sigma^2}\right) \right).$$
\end{lem}
\begin{cor}\label{cor:KLG}
Let $\mu, \tilde \mu, \sigma, \tilde \sigma, x \in \mathbb{R}$. If $\sigma\tilde\sigma \ne 0$, then $$\dkl{\mathcal N(\mu,\sigma^2)}{\mathcal N(\tilde\mu,\tilde\sigma^2)} \leq \frac12  \cdot \left(\frac{(\mu-\tilde\mu)^2}{\sigma^2} +  \left( \frac{\tilde\sigma^2}{\sigma^2} -1 \right)^2 \cdot \min\left\{ 1, \frac{1}{6}\left(2+\frac{\sigma^2}{\tilde\sigma^2}\right) \right\}\right) \cdot \frac{\sigma^2}{\tilde\sigma^2}.$$
\end{cor}
\begin{proof}
This follows from Lemma \ref{lem:KLG} and the inequalities $x - 1  - \ln x \leq (1/x-1)^2 \cdot x$ and $x - 1  - \ln x \leq (1/x-1)^2 \cdot x \cdot (2+x)/6$ for all $x > 0$.

\end{proof}

An analogous result holds for the Laplace distribution. Although we do not work this out, it implies that our results can be extended to work for the Laplace distribution (with slightly different constants and a higher power of $\ln k$, since the Laplace distribution has heavier tails).
\newcommand{\Lap}[2]{{\ifx&#1& \else #1 + \fi} \mathsf{Lap}\left(#2\right)}
Recall that the Laplace distribution with mean $\mu$ and variance $2\sigma^2$ --- denoted $\Lap{\mu}{\sigma}$ --- has a probability density at $x$  given by $\frac{1}{2|\sigma|}\exp\left(\left|\frac{x-\mu}{\sigma}\right|\right)$.
\begin{lem}[{\citep{GilAL13}, Table 3}]\label{lem:laplace-kl}
Let $\mu,\tilde\mu,\sigma,\tilde\sigma\in\mathbb{R}$. If $\sigma,\tilde\sigma > 0$, then \begin{align*}
\dkl{\Lap{\mu}{\sigma}}{\Lap{\tilde\mu}{\tilde\sigma}}
=& \frac{\sigma}{\tilde\sigma}\left(e^{-\frac{|\tilde\mu-\mu|}{\sigma}}-\left(1-\frac{|\tilde\mu-\mu|}{\sigma}\right)\right)+\frac{\sigma}{\tilde\sigma}-1-\ln\left(\frac{\sigma}{\tilde\sigma}\right)\\
\leq& \frac{(\tilde\mu-\mu)^2}{2\sigma\tilde\sigma}+\frac{1}{7}\left(\frac{\tilde\sigma^2}{\sigma^2}-1\right)^2 \frac{\sigma^2}{\tilde\sigma^2}.
\end{align*}
\end{lem}

Next we have a technical lemma relating expectations to KL divergence.
\begin{lem}[{\citep{Gray11}, Theorem 5.2.1}] \label{lem:sup}
Let $\cP$ and $\cQ$ be probability distributions on $\Omega$. Then $$\dkl{\cP}{\cQ} = \sup_{f : \Omega \to \mathbb{R}} \ex{X \sim \cP}{f(X)} - \ln \ex{X \sim \cQ}{e^{f(X)}}.$$
\end{lem}
Setting $f(x)=tx$ and rearranging gives the bound we will use:
\begin{cor}\label{cor:MGF-KL}
Let $X$ and $Y$ be real-valued random variables and $t>0$. Then $$\ex{}{X} \leq \frac{1}{t} \left( \dkl{X}{Y} + \ln \ex{}{e^{tY}} \right).$$
\end{cor}

Next we note that KL divergence is a convex function.
\begin{lem}[{\citep{VanErwenH14}, Theorem 11}]\label{lem:klconv}
Let $\cP_0$, $\cP_1$, $\cQ_0$, $\cQ_1$ be probability distributions on the same space $\Omega$. For $t \in (0,1)$, let $\cP_t=(1-t)\cP_0+t\cP_1$ and $\cQ_t=(1-t)\cQ_0+t\cQ_1$ be the convex combinations interpolating between these distributions. Then, for all $t \in [0,1]$, $$\dkl{\cP_t}{\cQ_t} \leq (1-t) \dkl{\cP_0}{\cQ_0} + t \dkl{\cP_1}{\cQ_1}.$$
\end{lem}
This lemma immediately extends to convex combinations of more than two distributions.

Next we have a geometric statement about KL divergence:

\begin{lem}[{\citep{BanerjeeMDG05} Proposition 1 \& \citep{FrigyikSG08}, Theorem II.1}] \label{lem:klmin} 
Let $\{\cP_y\}$ be a family of distributions indexed by $y \in \cY$ and let $\cQ$ be a distribution on $\cY$. Let $\cP_\cQ=\ex{Y \sim \cQ}{\cP_Y}$ denote the convex combination of the distributions $\{\cP_y\}$ weighted by $\cQ$. Then $$\inf_\cR \ex{Y \sim \cQ}{\dkl{\cP_Y}{\cR}} = \ex{Y \sim \cQ}{\dkl{\cP_Y}{\cP_\cQ}}.$$
\end{lem}
Lemma \ref{lem:klmin} shows that the ``center'' of a collection of probability distributions --- as measured my minimizing average KL divergence to one distribution --- is none other than the mean of those distributions.

\subsection{Mutual Information}

A key quantity that we use is mutual information:

\begin{defn}[Mutual Information]
For two random variables $X$ and $Y$ jointly distributed according to a distribution $\cP$ over $\X \times \cY$, the mutual information between $X$ and $Y$ is $$I(X;Y) = \dkl{ \cP(x,y)}{ \cP(x) \times \cP(y)} = \ex{x' \sim \cP(x)}{\dkl{\cP(y|x=x'}{\cP(y)}},$$ where  $\cP(x) \times \cP(y)$ denotes the product of the marginal distributions of $\cP$.
\end{defn}
Note that mutual information is symmetric -- $I(X;Y)=I(Y;X)$.
\begin{defn}[Conditional Mutual Information]
For three random variables $X$, $Y$, and $Z$. The mutual information between $X$ and $Y$ conditioned on $Z$ is given by
$$I(X;Y|Z) = \ex{z \sim \cP_Z}{I(X|Z=z;Y|Z=z)},$$ where  $\cP_Z$ is the marginal distribution of $Z$.
\end{defn}
The key property is the chain rule:
\begin{lem}[Mutual Information Chain Rule]\label{lem:MIchain}
For random variables $X$, $Y$, and $Z$, we have
$$I(X,Y;Z) = I(X;Z) + I(Y;Z|X)$$
\end{lem}

\section{Average KL Stability \& Generalization}

In this section, we cover our theoretical tools, which center around our definition of average leave-one-out KL stability, which we restate here. For a randomized algorithm $M$ and input $s$ we use $M(s)$ denote the random variable obtained by running on $M$ on $s$ and the distribution of this random variable (according to the context).

\begin{defn}[Average Leave-one-out KL stability]\label{defn:klas}
An algorithm $M : \left(\X^n \cup \X^{n-1} \right) \to \mathcal{Y}$ is $\varepsilon$-\KLAS{} if, for all $s \in \X^n$, $$\frac{1}{n} \sum_{i \in [n]} \dkl{M(s)}{M(s_{-i})} \leq \varepsilon,$$ where $s_{-i} \in \X^{n-1}$ denotes $s$ with the $i^\text{th}$ element removed.

More generally, an algorithm $M : \X^n \to \mathcal{Y}$ is $\varepsilon$-\KLAS{} if for every $i\in [n]$ there exists an algorithm $M_i : \left(\X^{n-1} \right) \to \mathcal{Y}$ such that under the same conditions $$\frac{1}{n} \sum_{i \in [n]} \dkl{M(s)}{M_i(s_{-i})} \leq \varepsilon .$$
\end{defn}
Note that the second definition extends the notion to mechanisms that are only defined for inputs in $\X^n$. It is also potentially weaker than the first condition. All the properties will hold under this weaker definition. It is easy to see that KL stability under replacement of a single element implies ALKL stability. Recall, that an algorithm $M$ is $\varepsilon$-KL stable (or, equivalently, $(1,\varepsilon)$-RDP \citep{Mironov17}) if for all $s,s'\in \X^n$ that differ in a single element, $\dkl{M(s)}{M(s')} \leq \varepsilon$ \citep{BassilyNSSSU16}. For every $i\in [n]$, $z\in \X^{n-1}$ and a fixed $x_0\in \X$, we can define $z \circ_i x_0$ as the vector $s\in \X^n$ such that $s_{-i} =z$ and $s_i=x_0$. Note that for every $i$ and $s\in \X^n$, $s$ and $s_{-i}\circ_i x_0$ differ in a single element. Therefore by defining $M_i(z) \doteq M(z\circ_i x_0)$, we obtain that for every $i$, $\dkl{M(s)}{M_i(s_{-i})} \leq \varepsilon$ and, in particular, $M$ also has $\varepsilon$-ALKL stability.
\begin{cor}
If an algorithm $M$ is $\varepsilon$-KL stable then it is $\varepsilon$-\KLAS{}.
\end{cor}


The key property of our definition is composition. This lemma allows us to account for the accumulation of information through multiple adaptive queries. The following lemma only considers the composition of two algorithms. Induction allows this to be extended to $k$ algorithms.

\begin{lem}[Composition]\label{lem:composition}
Suppose $M : \left(\X^n \cup \X^{n-1} \right) \to \mathcal{Y}$ is $\varepsilon$-\KLAS{} and $M' : \mathcal{Y} \times \left(\X^n \cup \X^{n-1} \right) \to \mathcal{Z}$ is such that $M'(y,\cdot) : \left(\X^n \cup \X^{n-1} \right) \to \mathcal{Z}$ is $\varepsilon'$-\KLAS{} for all $y \in \mathcal{Y}$. Then the composition $s \mapsto M'(M(s),s)$ is $(\varepsilon+\varepsilon')$-\KLAS{}.
\end{lem}
\begin{proof}
Let $M_i$ and $M_i'$ be the algorithms whose existence is assumed by Definition \ref{defn:klas}. Fix $s \in \mathcal{X}^n$. By the chain rule for KL divergence (Lemma \ref{lem:KLchain}),
\begin{align*}
&\frac{1}{n} \sum_{i \in [n]} \dkl{M'(M(s),s)}{M_i'(M_i(s_{-i}),s_{-i})}\\
&\leq \frac{1}{n} \sum_{i \in [n]} \left(\dkl{M(s)}{M_i(s_{-i})} + \ex{y \sim M(s)}{\dkl{M'(y,s)}{M_i'(y,s_{-i})}}\right)\\
&= \frac{1}{n} \sum_{i \in [n]} \dkl{M(s)}{M_i(s_{-i})} + \ex{y \sim M(s)}{ \frac{1}{n} \sum_{i \in [n]}\dkl{M'(y,s)}{M_i'(y,s_{-i})}}\\
&\leq \varepsilon + \varepsilon'.
\end{align*}
\end{proof}

Another key property of our definition of average leave-one-out KL stability is \emph{postprocessing}. That is, if $M$ is $\varepsilon$-\KLAS{}, then applying an arbitrary function to the output of $M$ continues to be $\varepsilon$-\KLAS{}. This can be seen by taking $\varepsilon'=0$ in the above composition lemma or by using the data processing inequality for KL divergence \cite[Theorem 1]{VanErwenH14}.

\subsection{Mutual Information}
In order to show that our notion of average leave-one-out KL stability implies generalization, we first show that it implies a bound on mutual information:

\begin{prop} \label{prop:mi}
Let $M : \X^n \to \mathcal{Y}$ be $\varepsilon$-\KLAS{}. Let $S \in \X^n$ be a product distribution. Then $I(M(S);S) \leq \varepsilon \cdot n$.
\end{prop}

To prove Proposition \ref{prop:mi}, we introduce an intermediate notion of stability that is based on that of \citetall{RaginskyRTWX16}. Specifically, mutual information stability is defined as follows.
\begin{defn}[Restating Definition~\ref{def:mis}]
A randomized algorithm $M : \X^n \to \mathcal{Y}$ is $\varepsilon$-\MIS{} if, for any random variable $S$ distributed over $\X^n$ (including non-product distributions), $$\frac{1}{n} \sum_{i=1}^n I(M(S);S_i|S_{-i}) \leq \varepsilon.$$
\end{defn}
We show that mutual information stability has the following properties.
\begin{enumerate}
\item Average leave-one-out KL stability implies mutual information stability.
\item Mutual information stability implies a mutual information bound.
\item Mutual information stability composes adaptively.
\end{enumerate}
 Combining properties (1) and (2) yields Proposition \ref{prop:mi}. The adaptive composition property of mutual information stability implies that it might be useful for analysis of adaptive procedures which are not \KLAS~(although we do not use this property since ALKL stability itself composes adaptively).

\begin{lem} \label{lem:klas-mis}
If $M : \X^n \to \mathcal{Y}$ is $\varepsilon$-\KLAS{}, then it also is $\varepsilon$-\MIS{}.
\end{lem}
\begin{proof}
Let $S$ be a random variable distributed according to some distribution $\cP$ on $\X^n$. Let $\cP_i$ and $\cP_{-i}$ denote the marginal distribution of $S_i$ and $S_{-i}$, respectively. For $z \in \X^{n-1}$ we use $\cP(s_i|z)$ to denote the distribution of $S_i$ conditioned on $S_{-i} = z$. Now, by the definition of (conditional) mutual information,
\begin{align*}
\frac{1}{n} \sum_{i=1}^n I(M(S);S_i|S_{-i}) =& \frac{1}{n} \sum_{i=1}^n \ex{z \sim \cP_{-i}}{I(M(S)|S_{-i}=z;S_i|S_{-i}=z)}\\
=& \frac{1}{n} \sum_{i=1}^n \ex{z \sim \cP_{-i}}{\ex{x \sim \cP(s_i|z)}{\dkl{M(S)|S_{-i}=z,S_i=x}{M(S)|S_{-i}=z}}} \\
=& \frac{1}{n} \sum_{i=1}^n \ex{z \sim \cP_{-i}}{\ex{x \sim \cP(s_i|z)}{\dkl{M(z\circ_i x)}{M(S)|S_{-i}=z}}}.
\end{align*}
Here $z \circ_i x$ refers to the vector $s\in \X^n$ such that $s_{-i} =z$ and $s_i=x$. Here the inner expectation is over $x$ drawn from the distribution of $S_i$ conditioned on $S_{-i} = z$ --- of the KL divergence from the distribution of $M(z\circ_i x)$ to the distribution of $M(S)$ conditioned on $S_{-i}=z$. The latter distribution is exactly the convex combination of the distribution of $M(z\circ_i x)$ weighted by the distribution of $x \sim \cP(s_i|z)$.

Now the key observation: the convex combination --- $M(S)|S_{-i}=z$ --- is the distribution that minimizes the inner expectation. Hence, we can replace it by $M_i(z)$ and only increase the expression.  Formally, by Lemma \ref{lem:klmin}, for all $z\in \X^{n-1}$, $$\ex{x\sim \cP(s_i|z)}{\dkl{M(z\circ_i x)}{M(S)|S_{-i}=z)}} \leq \ex{x\sim \cP(s_i|z)}{\dkl{M(z\circ_i x)}{M_i(z)}}$$
Note that $M_i(z)$ refers either to execution of $M$ itself or (if $M$ is not defined over inputs of length $n-1$) to the algorithms whose existence is promised by the second half of Definition \ref{defn:klas}. The result now follows, as we have established that
\begin{align*}
\frac{1}{n} \sum_{i=1}^n I(M(S);S_i|S_{-i}) \leq& \frac{1}{n} \sum_{i=1}^n \ex{z \sim \cP_{-i}}{\ex{x \sim \cP(s_i|z)}{\dkl{M(z\circ_i x)}{M_i(z)}}}\\
=&  \frac{1}{n} \sum_{i=1}^n \ex{s \sim \cP}{\dkl{M(s)}{M_i(s_{-i})}}.
\end{align*}
\end{proof}

\begin{lem} \label{lem:mis-mi}
Suppose $M : \X^n \to \mathcal{Y}$ is $\varepsilon$-\MIS{}. Let $\cP$ be a distribution over $\X$ and $S$ be distributed according to $\cP^n$. Then $I(M(S);S) \leq \varepsilon n$.
\end{lem}
\begin{proof} 
Denote $S_{<i}=(S_1,\cdots, S_{i-1})$, $S_{>i}=(S_{i+1},\cdots, S_n)$ and $S_{\leq i}=(S_1,\cdots, S_{i})$. By the chain rule for mutual information (Lemma \ref{lem:MIchain} and induction), $$I(M(S);S) = \sum_{i=1}^n I(M(S);S_i|S_{<i}).$$
By the definition of (conditional) mutual information,
\alequ{ I(M(S);S_i|S_{<i}) &= \ex{z \sim \cP^{i-1}}{I(M(S)|S_{<i}=z;S_i|S_{< i} = z)} \nonumber\\
& = \ex{z \sim \cP^{i-1}}{\ex{x\sim \cP}{\dkl{M(S)|S_{\leq i}=z\circ x}{M(S)|S_{< i}=z}}} \label{eq:mi2kl}
}
Here $z \circ x  \in \X^i$ denotes the concatenation of $z \in \X^{i-1}$ with $x \in \X$.
Now, by the convexity of KL divergence (Lemma \ref{lem:klconv}), we can move the randomness of $S_{>i}$ from the divergence and into the expectation.  Namely,
$$\dkl{M(S)|S_{\leq i}=z\circ x}{M(S)|S_{<i}=z} \leq \ex{z' \sim \cP^{n-i}}{\dkl{M(S)|S=z\circ x \circ z'}{M(S)|S_{-i}=z\circ z'}}.$$
Here we use the fact that $M(S)|S_{\leq i}=z\circ x$ and $M(S)|S_{<i}=z$ are convex combinations of the distribution of $M(s)$ weighted by $S_{>i}$ (note that independence is crucial here).
Plugging this into eq.~\eqref{eq:mi2kl} and using the definition of (conditional) mutual information we get
\alequn{ I(M(S);S_i|S_{<i}) &\leq \ex{z \sim \cP^{i-1}}{\ex{x\sim \cP}{\ex{z' \sim \cP^{n-i}}{\dkl{M(S)|S=z\circ x \circ z'}{M(S)|S_{-i}=z\circ z'}}}} \\
& = \ex{s \sim \cP}{\dkl{M(S)|S=s}{M(S)|S_{-i}=s_{-i}}} \\
& = I(M(S);S_i|S_{-i})}

Combining these (in)equalities yields the result: $$I(M(S);S) \leq \sum_{i=1}^n I(M(S);S_i|S_{-i}).$$
\end{proof}

\begin{lem}\label{lem:mis-comp}
Suppose $M : \X^n \to \mathcal{Y}$ is $\varepsilon$-\MIS{} and $M' : \mathcal{Y} \times \X^n \to \mathcal{Z}$ is such that $M'(y,\cdot) : \X^n \to \mathcal{Z}$ is $\varepsilon'$-\MIS{} for all $y \in \mathcal{Y}$. Then the composition $s \mapsto M'(M(s),s)$ is $(\varepsilon+\varepsilon')$-\MIS{}.
\end{lem}
\begin{proof}
Let $S$ be a random variable on $\X^n$ and let $\cP_{M(S)}$ denote the probability distribution of $M(S)$. By the chain rule,
\begin{align*}
\frac{1}{n} \sum_{i=1}^n I(M'(M(S),S);S_i|S_{-i})\leq & \frac{1}{n} \sum_{i=1}^n I(M(S);S_i|S_{-i}) + I(M'(M(S),S);S_i|S_{-i},M(S))\\
=& \frac{1}{n} \sum_{i=1}^n I(M(S);S_i|S_{-i}) + \ex{y \sim \cP_{M(S)}}{I(M'(y,S);S_i|S_{-i},M(S)=y)}\\
=& \frac{1}{n} \sum_{i=1}^n I(M(S);S_i|S_{-i}) + \ex{y \sim \cP_{M(S)}}{\frac{1}{n} \sum_{i=1}^n I(M'(y,S);S_i|S_{-i},M(S)=y)}\\
\leq& \varepsilon + \varepsilon'.
\end{align*}
The key is that the stability property holds for all distributions, which means it holds for the distribution of $S$ conditioned on $M(S)$. Note that if we defined mutual information stability only to quantify over product distributions, then this proof would not carry through, as $S$ conditioned on $M(S)$ is not necessarily a product distribution anymore.
\end{proof}

\begin{rem}
A natural question to ask is whether instead of using stability notions we can directly use mutual information for our analysis. Specifically, one could prove a bound on the mutual information of adding calibrated noise and then use composition properties of mutual information to bound the error of the entire algorithm for answering adaptive queries. For this approach to work one needs to prove a bound on the mutual information of adding calibrated noise for arbitrary input distributions (or at least for all distributions that might result from conditioning on the previous answers to queries). However, it is not hard to see that for non-product distributions over $S$ mutual information can be much larger than the bounds we will get via stability. (For example if the distribution on $S$ is such that the answer to the query is $0$ with probability $1/2$ and $1$ with probability $1/2$ then adding noise with variance $1$ will reveal some positive constant amount of information. At the same time this algorithm is $1/n^2$-KL stable so in our approach will contribute only $1/n$ to the final bound on mutual information.) As a result, this simpler approach is unlikely to lead to useful generalization bounds.
\end{rem}

\subsection{Generalization in expectation}
In this section we translate an upper bound on $I(S;M(S))$ into an upper bound on the expectation of the generalization error. As in earlier work  \citep{RussoZ16}, our main technical tool is Corollary~\ref{cor:MGF-KL}. However we deal with more general random variables (not just subgaussian) and also prove bounds that are scaled to standard deviation of the random variable as opposed to the subgaussian constant. In Section.~\ref{sec:events} we describe an alternative approach to generalization which is based on bounding the probability of any ``bad'' event.

The following proposition bounds the expected generalization error.
\begin{prop}\label{prop:gen}
Let $M$ be a randomized algorithm with input from $\mathcal{X}^n$ and output in $\mathcal{Q}$, where $\mathcal{Q}$ is the set of functions $\psi : \mathcal{X} \to [0,1]$. Let $\cP$ be a distribution on $\mathcal{X}$ and $S \sim \cP^n$. Let $\tau>0$. Suppose $I(S;M(S)) \leq \varepsilon n$. Then $$\ex{S \sim \cP^n \atop \psi \sim M(S)}{\frac{S[\psi]-\cP[\psi]}{\mymax{\sdv(\psi(\cP))}{\tau}}} \leq \left\{ \begin{array}{cl} 2\sqrt{\varepsilon} & \text{ if } \sqrt\varepsilon \leq \tau \\ \varepsilon/\tau + \tau& \text{ if } \sqrt\varepsilon \geq \tau\end{array} \right\} \leq 2\sqrt\varepsilon + \varepsilon/\tau.$$
\end{prop}
\begin{proof}
Define a random variable $X=\frac{S[\psi]-\cP[\psi]}{\mymax{\sdv(\psi(\cP))}{\tau}}$ for $S \sim \cP^n$ and $\psi \sim M(S)$. Our goal is to bound $\ex{}{X}$. Let $Y=\frac{S[\psi]-\cP[\psi]}{\mymax{\sdv(\psi(\cP))}{\tau}}$ for $(S,S') \sim \cP^n \times \cP^n$ and $\psi \sim M(S')$. That is, $Y$ is $X$ altered so that the query $\psi$ is independent of the data $S$, but has the same marginal distribution. Since $I(S;M(S)) = \dkl{S,M(S)}{S,M(S')} \leq \varepsilon n$, we have $\dkl{X}{Y} \leq \varepsilon n$ by the data processing inequality. By Corollary \ref{cor:MGF-KL}, \begin{equation}\ex{}{X} \leq \inf_{\lambda>0} \frac{1}{\lambda} \left( \dkl{X}{Y} + \ln \ex{}{e^{\lambda Y}} \right).\label{eqn:kl-mgf1}\end{equation}
Thus it only remains to bound $\ex{}{e^{\lambda Y}}$. We have
\begin{align*}
\ex{}{e^{\lambda Y}}
&= \ex{(S,S') \sim \cP^n \times \cP^n \atop \psi \sim M(S')}{\exp\left(\frac{\lambda}{n} \sum_{i =1}^{n} \frac{\psi(S_i)-\cP[\psi]}{\mymax{\sdv(\psi(\cP))}{\tau}}\right)}\\
&= \ex{S' \sim \cP^n \atop \psi \sim M(S')}{\prod_{i=1}^n \ex{S_i \sim \cP }{\exp\left(\frac{\lambda}{n} \frac{\psi(S_i)-\cP[\psi]}{\mymax{\sdv(\psi(\cP))}{\tau}}\right)}}.
\end{align*}
Thus it suffices to bound $\ex{S_i \sim \cP }{\exp\left(\frac{\lambda}{n} \frac{\psi(S_i)-\cP[\psi]}{\mymax{\sdv(\psi(\cP))}{\tau}}\right)}$ for a fixed $i$ and a fixed $\psi$. The random variable $Y_i=\frac{\psi(S_i)-\cP[\psi]}{\mymax{\sdv(\psi(\cP))}{\tau}}$ has mean $0$ and variance at most $1$ (since $S_i \sim \cP$). Also $|Y_i| \leq 1/\tau$. Thus by Lemma \ref{lem:BernsteinMGF} (stated below), we have $\ex{}{e^{\frac{\lambda}{n} Y_i}} \leq e^{\frac{\lambda^2}{n^2}}$ for all $\lambda/n\leq \tau$. Hence $\ex{}{e^{\lambda Y}} \leq e^{\lambda^2/n}$ for all $\lambda \leq n\tau$. Plugging this into eq.~\eqref{eqn:kl-mgf1}, we get $$\ex{}{X} \leq \inf_{0 < \lambda \leq n\tau} \frac{\varepsilon n}{\lambda}+\frac{\lambda}{n} = \left\{ \begin{array}{cl} 2\sqrt{\varepsilon} & \text{ if } \sqrt\varepsilon \leq \tau \\ \varepsilon/\tau + \tau& \text{ if } \sqrt\varepsilon \geq \tau\end{array} \right\}.$$
\end{proof}
\begin{lem}\label{lem:BernsteinMGF}
Let $Y$ be a random variable supported on $[-1/\tau,1/\tau]$. Suppose $\ex{}{Y}=0$ and $\ex{}{Y^2} \leq 1$. Then for $\lambda \in [0,\tau]$, $\ex{}{e^{\lambda Y}} \leq e^{\lambda^2}$.
\end{lem}
This lemma is similar to the proof of Bernstein's inequality \citep{Bernstein24}.
\begin{proof}
Since $|Y| \leq 1/\tau$, we have $|Y|^k \leq (1/\tau)^{k-2} Y^2$ for all $k \geq 2$. Thus, for all $\lambda \geq 0$, we have
\alequn{
\ex{}{e^{\lambda Y}} &= 1 + \lambda \ex{}{Y} + \sum_{k=2}^\infty \frac{\lambda^k}{k!} \ex{}{Y^k} \leq 1 + \sum_{k=2}^\infty \frac{\lambda^k}{k!} (1/\tau)^{k-2} \ex{}{Y^2} \\ &= 1 +  \left(e^{\lambda/\tau} - 1 -\lambda/\tau\right)\tau^2\ex{}{Y^2} \leq e^{\left(e^{\lambda/\tau}-1-\lambda/\tau\right) \cdot \tau^2}.}
If $\lambda/\tau \leq 1$, then $e^{\lambda/\tau} - 1-\lambda/\tau \leq (\lambda/\tau)^2$, which yields the result.
\end{proof}

Proposition \ref{prop:gen} gives a bound in terms of variance. Using the PAC-Bayes framework, we can also attain an additive-multiplicative bound:
\begin{prop}[\citep{McAllester13} Theorem 3]\footnote{Note that \citet{McAllester13} does not state this bound in terms of mutual information, but in terms of an equivalent KL divergence. Also, this result implies a two-sided bound by considering $1-\psi$.}Let $M$ be a randomized algorithm that takes an input from $\mathcal{X}^n$ and outputs a function $\psi : \mathcal{X} \to [0,1]$. Let $\mathcal{P}$ be a distribution on $\mathcal{X}$ and let $S \sim \mathcal{P}^n$ consist of $n$ i.i.d.~samples therefrom. Fix $\lambda>1/2$. Then $$\ex{S \sim \mathcal{P}^n \atop \psi \sim M(S)}{\mathcal{P}[\psi]} \leq \frac{1}{1-1/2\lambda}\left(\ex{S \sim \mathcal{P}^n \atop \psi \sim M(S)}{S[\psi]} + \frac{\lambda}{n} I(S;M(S))\right)$$
\end{prop}

To analyse our algorithm, we also need to bound the empirical error in terms of the standard deviation. Note that the empirical error -- the noise we add -- scales with the empirical standard deviation. Thus we must bound the empirical variance in terms of the true variance:
\begin{prop}\label{prop:gen-emp}
Let $M$ be a randomized algorithm with input from $\mathcal{X}^n$ and output in $\mathcal{Q}$, where $\mathcal{Q}$ is the set of functions $\psi : \mathcal{X} \to [0,1]$. Let $\cP$ be a distribution on $\mathcal{X}$ and $S \sim \cP^n$. Let $\tau>0$. Suppose $I(S;M(S)) \leq \varepsilon n$. Then $$\ex{S \sim \cP \atop \psi \sim M(S)}{\left(\frac{\sdv(\psi(S))}{\mymax{\sdv(\psi(\cP))}{\tau}}\right)^2} = \ex{S \sim \cP \atop \psi \sim M(S)}{\left(\frac{\sqrt{\frac{1}{n} \sum_{i=1}^n (\psi(S_i) - S[\psi])^2}}{\mymax{\sdv(\psi(\cP))}{\tau}}\right)^2} \leq 2+\varepsilon/\tau^2.$$
\end{prop}

To prove Proposition \ref{prop:gen-emp} we make use of the following two standard facts.

Let $\psi : \mathcal{X} \to \mathbb{R}$. Let $\cP$ be a distribution on $\mathcal{X}$ and $S \in \mathcal{X}^n$. Then
\begin{equation}\frac{1}{n} \sum_{i=1}^n (\psi(S_i) - S[\psi])^2 \leq \frac{1}{n} \sum_{i=1}^n\left(\psi(S_i)-\cP[\psi]\right)^2.\label{eqn:muddle}\end{equation}

Let $X$ be a random variable supported on $[0,1]$. Suppose $\ex{}{X}=\sigma$. Then for $s \in [0,1]$, \begin{equation}\ex{}{e^{sX}} \leq \ex{}{1+2sX}=1+2\sigma s \leq e^{2 \sigma s}.\label{eqn:BernsteinMGF-first}\end{equation}

\begin{proof}[Proof of Proposition \ref{prop:gen-emp}]
By \eqref{eqn:muddle},
$$\ex{S \sim \cP \atop \psi \sim M(S)}{\left(\frac{\sqrt{\frac{1}{n} \sum_{i=1}^n (\psi(S_i) - S[\psi])^2}}{\mymax{\sdv(\psi(\cP))}{\tau}}\right)^2} \leq \ex{S \sim \cP \atop \psi \sim M(S)}{\frac{\frac{1}{n} \sum_{i=1}^n (\psi(S_i) - \cP[\psi])^2}{\mymax{\var{}{\psi(\cP)}}{\tau^2}}}.$$
Define a random variable $X=\sum_{i=1}^n\frac{\left(\psi(S_i) - \cP[\psi]\right)^2}{\mymax{\var{}{\psi(\cP)}}{\tau^2}}$ for $S \sim \cP^n$ and $\psi \sim M(S)$. Our goal is thus to bound $\frac{1}{n}\ex{}{X}$. Define another random variable $Y=\sum_{i=1}^n\frac{\left(\psi(S_i) - \cP[\psi]\right)^2}{\mymax{\var{}{\psi(\cP)}}{\tau^2}}$ for $(S,S') \sim \cP^n \times \cP^n$ and $\psi \sim M(S')$. That is, $Y$ is defined for $S$ and $\psi$ being independent, whereas $X$ has them being dependent through $M$. By the data processing inequality, $\dkl{X}{Y} \leq \dkl{S,M(S)}{S,M(S')} = I(S;M(S)) \leq \varepsilon n$.
Now, by Corollary \ref{cor:MGF-KL},
\begin{equation}\ex{}{X} \leq \inf_{\lambda>0} \frac{1}{\lambda} \left( \dkl{X}{Y} + \ln \ex{}{e^{\lambda Y}} \right).\label{eqn:wowie}\end{equation}
To bound $\ex{}{e^{\lambda Y}}$ we note that $Y$ is determined by $S$ and $\psi$. Since these are independent, we may consider an arbitrary fixed $\psi$. We let $Z$ denote $Y$ conditioned on $\psi$ being equal to a fixed $\phi \in \cQ$.  We can write $Z$ as a sum of $n$ independent terms $Z_i=\frac{(\phi(S_i)-\cP[\phi])^2}{\mymax{\var{}{\phi(\cP)}}{\tau^2}}$, and hence $\ex{}{e^{\lambda Z}}=\prod_i^n \ex{}{e^{\lambda Z_i}}$. For each $i$, we have $\ex{}{Z_i} = \frac{\var{}{\phi(\cP)}}{\mymax{\var{}{\phi(\cP)}}{\tau^2}}\leq 1$ and $0 \leq Z_i \leq 1/\tau^2$. Thus, by \eqref{eqn:BernsteinMGF-first} (with $X=Z_i\tau^2$, $\sigma\leq \tau^2$, and $s=\lambda/\tau^2$), $\ex{}{e^{\lambda Z_i}} \leq e^{2\lambda}$ for $\lambda \in [0,\tau^2]$.

This implies that $\ex{}{e^{\lambda Z}} \leq e^{2\lambda n}$ for $\lambda \in [0,\tau^2]$ for every $\phi$ and hence $\ex{}{e^{\lambda Y}} \leq e^{2\lambda n}$ under the same condition. Substituting this into eq.~\eqref{eqn:wowie} yields $$\ex{}{X} \leq \inf_{0<\lambda\leq \tau^2} \frac{1}{\lambda} \left( \varepsilon n + \ln \left(e^{2\lambda n }\right) \right) = \varepsilon n/\tau^2 + 2n.$$
\end{proof}

\subsection{Preservation of low-probability events}
\label{sec:events}
Propositions \ref{prop:gen} and \ref{prop:gen-emp} bound the expected generalization error given a bound on mutual information.
An alternative approach to analysis of generalization is to use a bound on mutual information to upper bound the increase in probability of any ``bad'' event that results from the dependence between the dataset and algorithm's output. Specifically, we prove the following simple lemma:
\begin{lem}\label{lem:prob-gen}
Let $S$ consist of $n$ independent samples from some distribution $\cP$. Let $S'$ be an independent copy of $S$. Let $M : \mathcal{X}^n \to \mathcal{Y}$ and let $E$ be an event on $\mathcal{X}^n \times \mathcal{Y}$ satisfying $$\prob{}{(S',M(S))\in E} \leq \delta.$$ Then $$\prob{}{(S,M(S))\in E} \leq \frac{I(S;M(S)) + \ln 2}{\ln(1/\delta)}.$$
\end{lem}
Intuitively, Lemma \ref{lem:prob-gen} says that if an event happens with very low probability on fresh data, then it happens with somewhat low probability on non-fresh data, as long as the mutual information between the event and the data is low. Note however, that the probability grows from $\delta$ to $\frac{I(S;M(S)) + \ln 2}{\ln(1/\delta)}$. In particular, the inverse of the probability decreases exponentially. For example, Lemma \ref{lem:prob-gen} can be used to correct a $p$-value obtained under the assumption that the data is independent from the choice of the test (since $p$-value is the probability that a test statistic satisfies a chosen condition) \citep{RussoZ16,RogersRST16}.

The same approach to generalization is used in \citep{DworkFHPRR14:arxiv,DworkFHPRR15:arxiv,RogersRST16} for differential privacy and max-information and in \citep{RussoZ16,RogersRST16} for mutual information. The bound implicit in \citep{RussoZ16} is $\prob{}{(S,M(S))\in E} \leq \delta + \sqrt{\frac{I(S;M(S))}{\ln(1/(2\delta))}}$ which is asymptotically worse than our bound. The bound in \citep{RogersRST16} is derived by first using mutual information to bound approximate max-information \citep{DworkFHPRR15:arxiv}. Their approach yields the following bound $$\prob{}{(S,M(S))\in E} \leq \inf_{k \geq 0}\left( 2^k \cdot \delta + \frac{I(S;M(S))+0.54}{k}\right)$$ which is comparable to the bound in Lemma \ref{lem:prob-gen}.

As a more concrete application, we demonstrate how Lemma \ref{lem:prob-gen} can be used to derive a bound on the probability of generalization error being large (or, equivalently, to construct a valid confidence interval for the true expectation of a real-valued function).

Recall the setting of Proposition \ref{prop:gen}. Here $S$ consists of $n$ independent samples from $\cP$ and $M$ outputs a function $\psi : \X \to [0,1]$ and has $I(S;M(S)) \leq \varepsilon n$. By Bernstein's inequality, for $n$ fresh samples $S'$ (independent from $M(S)$), we have $$\prob{(S,S') \sim \cP^n \times \cP^n \atop \psi \sim M(S)}{\frac{S'[\psi]-\cP[\psi]}{\mymax{\sdv(\psi(\cP)}{\tau}} > t} \leq \exp\left(\frac{-t^2 n}{2+\frac{2}{3}\frac{t}{\tau}}\right)$$ for all $t>0$. Thus, by Lemma \ref{lem:prob-gen}, for all $t>0$, $$\prob{S \sim \cP^n \atop \psi \sim M(S)}{\frac{S[\psi]-\cP[\psi]}{\mymax{\sdv(\psi(\cP)}{\tau}} > t} \leq \frac{2+\frac{2}{3}\frac{t}{\tau}}{t^2} \cdot \left( \varepsilon + \frac{\ln 2}{n} \right).$$
In our application of Proposition \ref{prop:gen}, we have $\varepsilon = \tau^2 \geq 1/n$; setting $t=3\tau/\beta$ and simplifying yields $$\prob{S \sim \cP^n \atop \psi \sim M(S)}{\frac{S[\psi]-\cP[\psi]}{\mymax{\tau \cdot \sdv(\psi(\cP)}{\tau^2}} > \frac{3}{\beta}} \leq \beta.$$
Note that this bound appears to correspond to an application of Markov's inequality to the conclusion Proposition \ref{prop:gen}. However, since the random variable in question may take both positive and negative values, Markov's inequality cannot be applied. Namely, Lemma \ref{lem:prob-gen} corresponds to a strengthening of Proposition \ref{prop:gen} that bounds the expectation of the absolute value of the random variable. This approach can also be easily used to get a bound (based on Markov's inequality) on the tail of the largest error we state in Theorem \ref{thm:main-intro}. This follows from the fact that a high probability bound on this tail is easy to prove when the dataset is independent from the algorithm's answers.



\medskip

To prove Lemma \ref{lem:prob-gen}, we observe that for any random variable $S$ and any randomized algorithm $M$, $$\dkl{\ind_E(S,M(S))}{\ind_E(S',M(S))} \leq \dkl{S,M(S)}{S',M(S)} = I(S;M(S))$$ where $S'$ is an independent copy of $S$, $E$ is an arbitrary event on $\X^n\times \cY$, and $\ind_E$ is the indicator function of the event. Note that $\ind_E(S,M(S))$ is a Bernoulli random variable with bias equal to $\pr[(S,M(S))\in E]$. Now the proof of Lemma \ref{lem:prob-gen} follows from the following lemma.
\begin{lem}
Let $B(p)$ denote the Bernoulli random variable with bias $p$. Then for any $p,q \in (0,1]$, $$p \leq \frac{\dkl{B(p)}{B(q)}+\ln 2}{\ln(1/q)}.$$
\end{lem}
\begin{proof}
We have
\begin{align*}
\dkl{B(p)}{B(q)} &= p \ln\left(\frac{p}{q}\right) + (1-p) \ln\left(\frac{1-p}{1-q}\right)\\
&= p \ln(1/q) - \mathsf{H}(p) + (1-p) \ln(1/(1-q))\\
&\geq p \ln(1/q) - \ln 2 + 0,
\end{align*}
where $\mathsf{H}(p)=p\ln(1/p)+(1-p)\ln(1/(1-p))$ is the binary entropy function. Rearranging yields the result.
\end{proof}

\section{Analysis of our Algorithm}

Now we assemble the tools developed in the previous section to analyse our algorithm (Figure \ref{fig:scaledgauss}). To do this we must introduce some formalisms for dealing with adaptive algorithms.

Our algorithm $M$ answers adaptively-chosen queries. We call the entity $A$ choosing these queries the \emph{analyst} (or \emph{adversary} to connote worst-case behaviour). The interaction between $A$ and $M$ defines a function mapping inputs (the sample) to a transcript of queries and answers. Figure \ref{fig:interact}  defines how this function is computed.

\begin{figure}[h!]
\begin{framed}
\begin{algorithmic}
\INDSTATE[0]{Input $s \in \X^m$ is given to $M$.}
\INDSTATE[0]{For $j=1,2, \ldots, k$:}
\INDSTATE[1]{$A$ computes a query $\psi_j \in \cQ$ and passes it to $M$}
\INDSTATE[1]{$M$ produces answer $v_j \in \mathcal{R}$ and passes it to $A$}
\INDSTATE[0]{The output is the transcript $(\psi_1, \psi_2, \ldots, \psi_k, v_1, v_2, \ldots, v_k) \in \cQ^k \times \mathcal{R}^k$.}
\end{algorithmic}
\end{framed}
\caption{$\interact{A}{M} : \X^m \to \cQ^k \times \mathcal{R}^k$}\label{fig:interact}
\end{figure}

With this formalism in hand, we can extend our definition of average leave-one-out KL stability from non-interactive algorithms (Definition \ref{defn:klas-intro}) to interactive algorithms:

\begin{defn}[Interactive ALKL stability]\label{defn:klas-ada}
An interactive algorithm $M$ is $\varepsilon$-\KLAS{} if $\interact{A}{M}$ (as defined in Figure \ref{fig:interact}) is $\varepsilon$-\KLAS{} for all interactive algorithms $A$.
\end{defn}

\subsection{Stability of our algorithm}\label{sec:M}

We now show that our algorithm is (interactive) average leave-one-out KL stable:

\begin{thm}\label{thm:isklas}
Our algorithm (Figure \ref{fig:scaledgauss}) is $\frac{kt}{n^2}$-\KLAS{} for any $n \geq 20$ and $T \leq \min\{ t^2, tn/10 \}$.
\end{thm}

To establish that our algorithm is average leave-one-out KL stable, we first only consider one query $\psi=\psi_j$. We show that, for each query $\psi$, the answer given by our algorithm is $\frac{t}{n^2}$-\KLAS{}. Using composition (Lemma \ref{lem:composition-intro}), we can extend this to $k$ queries $\psi_1, \ldots, \psi_k$. That is, we prove that our algorithm is $\frac{kt}{n^2}$-\KLAS{}.

First we recall how our algorithm answers a query $\psi : \X \to [0,1]$: The algorithm is given as input a sample $s \in \X^n$ and, for each statistical query $\psi=\psi_j$ the algorithm $M$ outputs a sample from $\mathcal N(\mu,\mymax{\sigma^2/t}{1/T})$ where $$\mu= s[\psi] = \frac{1}{n} \sum_{i \in [n]} \psi(s_i) ~~~\text{ and }~~~ \sigma^2 = \var{}{\psi(s)} = \frac{1}{n} \sum_{i \in [n]} (\psi(s_i)-\mu)^2.$$
Here $t,T>0$ are parameters controlling the accuracy-stability tradeoff.

We also consider the following quantities in the analysis so that $M(s_{-i})$ outputs a sample from $\mathcal N(\mu_{-i},\mymax{\sigma_{-i}^2/t}{1/T})$. $$\mu_{-i}= s_{-i}[\psi] = \frac{1}{n-1} \sum_{j \in [n]\setminus\{i\}} \psi(s_{j}) ~~~\text{ and }~~~ \sigma_{-i}^2 = \var{}{\psi(s_{-i})}=\frac{1}{n-1} \sum_{j \in [n] \setminus\{i\}} (\psi(s_{j})-\mu_{-i})^2.$$

Before giving the full proof we give a simplified sketch.
\begin{proof}[Simplified Proof Sketch.]

We make three simplifications for our sketch of the analysis:
\begin{itemize}
\item Ignore constant factors. (Take $n$ and $\frac{t}{T}n$ to be sufficiently large.)
\item Consider $\psi : \X \to \{0,1\}$ instead of $\psi : \X \to [0,1]$.
\item Assume $\frac{\sigma^2}{t}\geq \frac{1}{T}$ and $\frac{\sigma_{-i}^2}{t}\geq \frac{1}{T}$ for all $i$.
\end{itemize}
The last assumption is not really an assumption, since we perforce ensure this by using $\mymax{\frac{\sigma^2}{t}}{\frac{1}{T}}$ in place of $\frac{\sigma^2}{t}$ and likewise for $\frac{\sigma_{-i}^2}{t}$. This assumption is only to simplify notation here in this sketch.

Begin by considering a fixed index $i \in [n]$.
By standard properties of the KL divergence between Gaussians (Corollary \ref{cor:KLG}) and assuming $\frac{\sigma^2}{\sigma_{-i}^2} \leq 2$, we have
\begin{align*}
&\dkl{\mathcal N \left(\mu, \mymax{\frac{\sigma^2}{t}}{\frac{1}{T}}\right)}{\mathcal N \left(\mu_{-i}, \mymax{\frac{\sigma_{-i}^2}{t}}{\frac{1}{T}}\right)}
=\dkl{\mathcal N \left(\mu, \frac{\sigma^2}{t}\right)}{\mathcal N \left(\mu_{-i}, \frac{\sigma_{-i}^2}{t}\right)}\\
&~~~= t\cdot\frac{(\mu-\mu_{-i})^2}{2\sigma_{-i}^2} + \frac12\left(\frac{\sigma^2}{\sigma_{-i}^2} - 1 - \ln\left(\frac{\sigma^2}{\sigma_{-i}^2}\right)\right)
\leq t\cdot\frac{(\mu-\mu_{-i})^2}{\sigma^2} + \left(\frac{\sigma_{-i}^2-\sigma^2}{\sigma^2}\right)^2. \numberthis\label{eqn:klgauss-main}
\end{align*}

Since we assumed (for simplicity) that $\psi$ takes only values $0$ and $1$, the distribution of $\psi(s_i)$ (for a random $i$) is characterized by its mean $\mu=s[\psi]$. In particular, we can express the variance as $\sigma^2 = \mu(1-\mu)$ and, likewise, $\sigma_{-i}^2 = \mu_{-i}(1-\mu_{-i})$.  Thus, we have $|\sigma_{-i}^2-\sigma^2| \leq |\mu-\mu_{-i}|$ and
\begin{equation}t\cdot\frac{(\mu-\mu_{-i})^2}{\sigma^2} + \left(\frac{\sigma_{-i}^2-\sigma^2}{\sigma^2}\right)^2 \leq \left( t + \frac{1}{\sigma^2} \right) \cdot\frac{(\mu-\mu_{-i})^2}{\sigma^2} \le 2t \cdot \frac{(\mu-\mu_{-i})^2}{\sigma^2},\label{eqn:sigma-mu}\end{equation} where the final inequality follows from the assumptions $\frac{\sigma^2}{t} \geq \frac{1}{T}$ and $T \le t^2$.
Also $$\mu-\mu_{-i} = \frac{1}{n} \left(\sum_{i'} \psi(s_{i'})\right) - \frac{1}{n-1} \left(-\psi(s_i)+\sum_{i'} \psi(s_{i'})\right) = \left(\frac{1}{n}-\frac{1}{n-1}\right) \cdot n\mu + \frac{\psi(s_i)}{n-1}=\frac{\psi(s_i)-\mu}{n-1}$$ and, hence, (now we make critical use of the averaging ALKL stability affords us) \begin{equation}\frac{1}{n}\sum_{i\in [n]} (\mu-\mu_{-i})^2 = \frac{1}{n}\sum_{i\in [n]} \left(\frac{\psi(s_i)-\mu}{n-1}\right)^2 = \frac{\sigma^2}{(n-1)^2}.\label{eqn:mu-sigma}\end{equation}
Combining the above equations (\ref{eqn:klgauss-main},\ref{eqn:sigma-mu},\ref{eqn:mu-sigma}) yields
$$
\frac{1}{n} \sum_{i\in[n]} \dkl{\mathcal N \left(\mu, \mymax{\frac{\sigma^2}{t}}{\frac{1}{T}}\right)}{\mathcal N \left(\mu_{-i}, \mymax{\frac{\sigma_{-i}^2}{t}}{\frac{1}{T}}\right)} ~~~~~~~~~~~~~~
$$
$$~~~~~~~~~~~~~~~\le \frac{1}{n} \sum_{i\in[n]} 2t \cdot \frac{(\mu-\mu_{-i})^2}{\sigma^2} = \frac{2t}{\sigma^2} \cdot \frac{\sigma^2}{(n-1)^2} \le O\left(\frac{t}{n^2}\right),$$
as desired to show $O(t/n^2)$-ALKL stability.
\end{proof}

The key facts that drive this simplified proof (and which hold without the simplifying assumptions) are the KL divergence between Gaussians \eqref{eqn:klgauss-main}, $$\frac{1}{n}\sum_{i\in [n]} (\mu-\mu_{-i})^2 = \frac{\sigma^2}{(n-1)^2}\leq O\left(\frac{\sigma^2}{n^2}\right), ~~~~~\text{and}~~~~~ \frac{1}{n}\sum_{i\in [n]} (\sigma^2-\sigma_{-i}^2)^2 \leq O\left(\frac{\sigma^2}{n^2}\right).$$ Indeed, we can redefine $\mu$ and $\sigma^2$ (e.g., to answer different types of queries, rather than statistical queries) and still retain ALKL stability, as long as the above two inequalities hold.

Now we prove the general result with sharp constants. Theorem \ref{thm:isklas} is implied by the following result.

\begin{prop}\label{prop:M}
Let $s \in \X^n$, $n \geq 2$, $t,T>0$, and $\psi : \X \to [0,1]$. For $i \in [n]$, define
\begin{align*}
\mu&=s[\psi]=\frac{1}{n} \sum_{i \in [n]} \psi(s_i), & \mu_{-i} &= s_{-i}[\psi] = \frac{1}{n-1} \sum_{j \in [n]\setminus\{i\}} \psi(s_{j}),\\
\sigma^2 &= \var{}{\psi(s)} = \frac{1}{n} \sum_{i \in [n]} (\psi(s_i)-\mu)^2, & \sigma_{-i}^2 &= \var{}{\psi(s_{-i})}=\frac{1}{n-1} \sum_{j \in [n] \setminus\{i\}} (\psi(s_{j})-\mu_{-i})^2.
\end{align*}
Then \begin{equation}\frac{1}{n} \sum_{i \in [n]} \dkl{\mathcal N \left(\mu, \mymax{\frac{\sigma^2}{t}}{\frac{1}{T}}\right)}{\mathcal N \left(\mu_{-i}, \mymax{\frac{\sigma_{-i}^2}{t}}{\frac{1}{T}}\right)} \leq \frac{1}{4n^2} \left( 2t + \frac{T}{t} \cdot \left( 1 + \zeta\right)\right) \cdot \left( 1 + \zeta\right),\label{eqn:AKL-bound}\end{equation}
where $1+\zeta=\left( 1 + \frac{1}{n-1}\right)^2\left( 1 + \frac{T}{tn}\left(1+\frac{1}{n-1}\right)^2\right)$.
\end{prop}
In particular, $\zeta=O\left(\frac{1}{n} + \frac{T}{tn}\right)$ and, if $n \geq 20$ and $\frac{T}{t} \leq \frac{n}{10}$, then $\eqref{eqn:AKL-bound} \leq \frac{1}{n^2} \max\{t,T/t\}$.

\begin{proof}

We have \begin{equation}\mu-\mu_{-i} = \mu - \frac{n\mu - \psi(s_i)}{n-1} = \frac{\psi(s_i)-\mu}{n-1}\label{eqn:mudiff}\end{equation} and \begin{equation}\frac{1}{n} \sum_{i \in [n]} \left( \mu - \mu_{-i} \right)^2 = \frac{1}{n(n-1)^2} \sum_{i \in [n]} \left( \psi(s_i)-\mu \right)^2 = \frac{\sigma^2}{(n-1)^2}.\label{eqn:musum}\end{equation}
Furthermore,
\begin{align*}
\sigma^2-\sigma_{-i}^2 =& \sigma^2 - \left(\frac{1}{n-1} \left(- \psi(s_i)^2  + \sum_{j \in [n]} \psi(s_{j})^2 \right) - \mu_{-i}^2\right)\\
=& \sigma^2 - \left(\frac{1}{n-1} \left(- \psi(s_i)^2  + n(\sigma^2+\mu^2) \right) - \mu_{-i}^2\right)\\
=& \sigma^2 \left( 1 - \frac{n}{n-1} \right) + \frac{\psi(s_i)^2}{n-1} - \frac{n}{n-1} \mu^2 + \mu_{-i}^2\\
=&  \frac{\psi(s_i)^2-\mu^2 - \sigma^2}{n-1} - (\mu^2 - \mu_{-i}^2) \\
=&  \frac{(\psi(s_i)-\mu)(\psi(s_i)+\mu) - \sigma^2}{n-1} - (\mu - \mu_{-i})(\mu + \mu_{-i}) \\
=&  \frac{(\psi(s_i)-\mu)(\psi(s_i)+\mu) - \sigma^2}{n-1} - \frac{\psi(s_i)-\mu}{n-1}(\mu + \mu_{-i}) ~~~~~\text{(by \eqref{eqn:mudiff})}\\
=&  \frac{\psi(s_i)-\mu}{n-1} \left((\psi(s_i)-\mu) + (\mu - \mu_{-i}) \right)  - \frac{\sigma^2}{n-1} \\
=&  \frac{\psi(s_i)-\mu}{n-1} \left((\psi(s_i)-\mu) + \frac{\psi(s_i)-\mu}{n-1}\right)  - \frac{\sigma^2}{n-1} ~~~~~\text{(by \eqref{eqn:mudiff})}\\
=& \frac{\frac{n}{n-1} (\psi(s_i)-\mu)^2 - \sigma^2}{n-1} \numberthis\label{eqn:sigmadiff}
\end{align*}
and
\begin{equation}
\left|\sigma^2-\sigma_{-i}^2\right| \leq  \frac{\max\{\frac{n}{n-1}(\psi(s_i)-\mu)^2 , \sigma^2 \}}{n-1} \leq \frac{n}{(n-1)^2}.\label{eqn:sigsens}
\end{equation}
By \eqref{eqn:sigmadiff},
\begin{align*}
\frac{1}{n} \sum_{i \in [n]} \left( \sigma^2-\sigma_{-i}^2\right)^2
=& \frac{1}{n(n-1)^2} \sum_{i \in [n]} \left(\frac{n}{n-1} (\psi(s_i)-\mu)^2 - \sigma^2\right)^2\\
=& \frac{1}{n(n-1)^2} \sum_{i \in [n]} \frac{n^2}{(n-1)^2} (\psi(s_i)-\mu)^4  - \frac{2n}{n-1} (\psi(s_i)-\mu)^2\sigma^2 + \sigma^4 \\
\leq& \frac{1}{n(n-1)^2} \sum_{i \in [n]} \frac{n^2}{(n-1)^2} (\psi(s_i)-\mu)^2  - \frac{2n}{n-1} (\psi(s_i)-\mu)^2\sigma^2 + \sigma^4 \\
&~~~~\text{(Since $0\leq (\psi(s_i)-\mu)^2 \leq 1$.)}\\
=& \frac{1}{(n-1)^2} \left( \frac{n^2}{(n-1)^2} \sigma^2  - \frac{2n}{n-1} \sigma^2\sigma^2 + \sigma^4 \right) \\
=& \frac{\sigma^2}{(n-1)^2} \left( \frac{n^2}{(n-1)^2} - \sigma^2 \frac{n+1}{n-1} \right) \\
\leq& \frac{\sigma^2}{(n-1)^2} \frac{n^2}{(n-1)^2}.\numberthis\label{eqn:sigmasum}
\end{align*}

Let $$\gamma=\max_i\left|\frac{\mymax{\frac{\sigma^2}{t}}{ \frac{1}{T}}}{\mymax{\frac{\sigma_{-i}^2}{t} }{ \frac{1}{T} }}-1\right|.$$
By \eqref{eqn:sigsens}, for some $i$, \begin{equation} \gamma =  \frac{\left|\mymax{\frac{\sigma_{-i}^2}{t} }{ \frac{1}{T} }-\mymax{\frac{\sigma^2}{t}}{\frac{1}{T}}\right|}{\mymax{\frac{\sigma_{-i}^2}{t}}{\frac{1}{T}}}  \leq \frac{\left| \frac{\sigma_{-i}^2}{t}-\frac{\sigma^2}{t}\right|}{\frac{1}{T}} = \frac{T}{t} |\sigma^2 - \sigma_{-i}^2| \leq \frac{Tn}{t(n-1)^2}.\label{eqn:ratio}\end{equation}

By Corollary \ref{cor:KLG}, we have
\begin{align*}
&\dkl{\mathcal N\left(\mu,\mymax{\frac{\sigma^2}{t}}{\frac{1}{T}}\right)}{\mathcal N\left(\mu_{-i},\mymax{\frac{\sigma_{-i}^2}{t}}{\frac{1}{T}}\right)}\\
&~~~\leq \frac12 \left( \frac{(\mu-\mu_{-i})^2}{\mymax{\frac{\sigma^2}{t}}{\frac{1}{T}}} + \left( \frac{\mymax{\frac{\sigma_{-i}^2}{t}}{ \frac{1}{T}}}{\mymax{\frac{\sigma^2}{t} }{ \frac{1}{T} }} - 1\right)^2 \cdot \min\left\{1, \frac{2+(1+\gamma)}{6}\right\} \right) \cdot (1+\gamma)\\
&~~~\leq \frac12 \left( \frac{(\mu-\mu_{-i})^2}{\mymax{\frac{\sigma^2}{t}}{\frac{1}{T}}} + \frac{\left( \mymax{\frac{\sigma_{-i}^2}{t}}{ \frac{1}{T}}-\mymax{\frac{\sigma^2}{t} }{ \frac{1}{T} }\right)^2}{\mymax{\frac{\sigma^2}{t} }{ \frac{1}{T} }^2} \cdot \frac12\left(1 + \frac{\gamma}{3}\right)\right) \cdot \left( 1 + \gamma \right)\\
&~~~\leq \frac12 \left( \frac{(\mu-\mu_{-i})^2}{\frac{\sigma^2}{t}} + \frac{\left( \frac{\sigma_{-i}^2}{t}-\frac{\sigma^2}{t} \right)^2}{\frac{\sigma^2}{t} \cdot  \frac{1}{T} } \cdot \frac{1+\gamma/3}{2} \right) \cdot \left( 1 + \gamma \right)
. \numberthis\label{eqn:klbound}
\end{align*}
Combining \eqref{eqn:klbound}, \eqref{eqn:musum},  \eqref{eqn:sigmasum}, and \eqref{eqn:ratio}, we have
\begin{align*}
&\frac{1}{n} \sum_{i \in [n]} \dkl{\mathcal N\left(\mu,\mymax{\frac{\sigma^2}{t}}{\frac{1}{T}}\right)}{\mathcal N\left(\mu_{-i},\mymax{\frac{\sigma_{-i}^2}{t}}{\frac{1}{T}}\right)}   \\
&~~~\leq \frac{1}{n} \sum_{i \in [n]} \frac12 \left( \frac{(\mu-\mu_{-i})^2}{\frac{\sigma^2}{t}} + \frac{\left( \sigma^2-\sigma_{-i}^2 \right)^2}{2 \cdot t^2 \cdot \frac{\sigma^2}{t} \cdot  \frac{1}{T} } (1+\gamma/3) \right) \cdot \left( 1 + \gamma \right)\\
&~~~\leq \frac12 \left( \frac{\frac{\sigma^2}{(n-1)^2}}{\frac{\sigma^2}{t}} + \frac{\frac{\sigma^2}{(n-1)^2}\frac{n^2}{(n-1)^2}}{2 \cdot t^2 \cdot \frac{\sigma^2}{t} \cdot  \frac{1}{T} } (1+\gamma/3)\right) \cdot \left( 1 + \gamma\right)\\
&~~~= \frac{1}{4(n-1)^2} \left( 2t + \frac{T}{t} \frac{n^2}{(n-1)^2} (1+\gamma/3)\right) \cdot \left( 1 + \gamma\right)\\
&~~~\leq \frac{1}{4(n-1)^2} \left( 2t + \frac{T}{t} \frac{n^2}{(n-1)^2} \left(1+\frac{Tn}{3t(n-1)^2}\right)\right) \cdot \left( 1 + \frac{Tn}{t(n-1)^2}\right)\\
&~~~= \frac{1}{4n^2} \left( 2t + \frac{T}{t} \cdot \left( 1 + \frac{1}{n-1}\right)^2 \left(1+\frac{Tn}{3t(n-1)^2}\right)\right) \cdot \left( 1 + \frac{1}{n-1}\right)^2\left( 1 + \frac{Tn}{t(n-1)^2}\right)\\
&~~~\leq \frac{1}{4n^2} \left( 2t + \frac{T}{t} \cdot \left( 1 + \zeta\right)\right) \cdot \left( 1 + \zeta\right)
.\numberthis\label{eqn:gammabound}
\end{align*}
\end{proof}


\subsection{Accuracy guarantees}
Note that, by the postprocessing property of average leave-one-out KL stability, applying any function $f$ to the transcript of an average leave-one-out KL stable algorithm still is average leave-one-out KL stable. More precisely, for a $\varepsilon$-\KLAS{} interactive algorithm $M$ and any interactive algorithm $A$, the composed algorithm mapping input $s$ to output $f(\interact{A}{M}(s))$ is $\varepsilon$-\KLAS{}. We now invoke the generalization properties of average leave-one-out KL stability:

\begin{lem}\label{lem:generrtrue}
Fix a distribution $\cP$ on $\X$. Let $M$ be a $\varepsilon$-\KLAS{} interactive algorithm that answers $k$ statistical queries. Let $A$ be an arbitrary interactive algorithm that asks $k$ statistical queries. Let $f : \cQ^k \times \R^k \to \cQ$, where $\cQ$ denotes the set of statistical queries $\psi : \X \to [0,1]$. Let $\tau=\sqrt{\varepsilon}$. Then \begin{equation}\left| \ex{S \sim \cP^n \atop \psi \sim f(\interact{A}{M}(S))}{\frac{S[\psi] - \cP[\psi]}{\mymax{\sdv(\psi(\cP))}{\tau}}} \right| \leq 2\sqrt\varepsilon \label{eqn:generr}\end{equation} and \begin{equation}\ex{S \sim \cP^n \atop \psi \sim f(\interact{A}{M}(S))}{\left(\frac{\sdv(\psi(S))}{\mymax{\sdv(\psi(\cP))}{\tau}}\right)^2}
\leq 3.\label{eqn:gen-emp}\end{equation}
\end{lem}

\begin{proof} This follows from our generalization results (Propositions \ref{prop:gen} and \ref{prop:gen-emp}), postprocessing, and the connection between average leave-one-out KL stability and mutual information (Proposition \ref{prop:mi}). \end{proof}

\medskip

A natural choice of function $f$ is to simply pick out one of the queries --- that is, $f(\psi_1, \ldots, \psi_k, v_1, \ldots, v_k) = \psi_j$ for some fixed $j$. However, $f$ can also pick out the ``worst'' query. For example, the monitor technique of \citet{BassilyNSSSU16} takes $$f_*(\psi_1, \ldots, \psi_k, v_1, \ldots, v_k) = \psi_{j_*}, ~~~~~\text{where}~~~~~j_* = \underset{j \in [k]}{\mathrm{argmax}} \left| v_j  - \cP[\psi_j] \right|.$$
The monitor technique allows us to reason about a single query and derive bounds that apply to all queries simultaneously as this query is the worst query. Since we have a more refined notion of error, we must use a slightly different argument.

We use the following function $f_{\tau,\cP} : \cQ^k \times \R^k \to \cQ $ to pick out the query with the worst scaled error. \begin{equation} f_{\tau,\cP}(\psi_1, \ldots, \psi_k, v_1, \ldots, v_k) = \left\{ \begin{array}{cl} \psi_{j_*} & \text{ if } v_{j_*}  \geq \cP[\psi_{j_*}] \\  1-\psi_{j_*} & \text{ if } v_{j_*}  < \cP[\psi_{j_*}]\end{array}\right\}, ~~~~~\text{where}~~~~~j_* = \underset{j \in [k]}{\mathrm{argmax}} \frac{\left| v_j  - \cP[\psi_j] \right|}{\mymax{\sdv(\psi_j(\cP)}{\tau}}.\end{equation}
Note that we flip the sign to ensure that the error is always positive. 

\medskip

Now we give a bound on the expected scaled error of our algorithm. We use the following simple technical lemma bounding the maximum of standard Gaussians.

\begin{lem}\label{lem:gaussmax}
Let $\xi_1, \ldots, \xi_k$ be independent samples from $\mathcal{N}(0,1)$. Then $$\ex{}{\max\{\xi_1^2, \ldots, \xi_k^2\}} \leq 2\ln (2k) .$$
\end{lem}
\begin{proof}
Let $X = \max_{i\in [k]} \xi_i^2$ and $t>0$. Since $\cosh(\sqrt{z})=\frac{1}{2}(e^{\sqrt{z}} + e^{-\sqrt{z}})$ is a convex function of $z \geq 0$, Jensen's inequality gives $$\cosh\left(t\sqrt{\ex{}{X}}\right) \leq \ex{}{\cosh\left(t\sqrt{X}\right)} = \ex{}{\max_j^k \cosh(t \xi_j) } \leq \ex{}{\sum_j^k \cosh(t\xi_j)} = k \cdot e^{t^2/2}.$$ Rearranging yields $$\ex{}{X} \leq \left( \frac{1}{t} \cosh^{-1}\left(k \cdot e^{t^2/2}\right) \right)^2 \leq \left( \frac{1}{t} \ln(2k \cdot e^{t^2/2}) \right)^2 = \left( \frac{\ln(2k)}{t} + \frac{t}{2} \right)^2,$$ as $\cosh(v) \geq e^v/2$ and, hence, $\cosh^{-1} (u) \leq \ln (2u)$. Setting $t=\sqrt{2\ln (2k)}$ completes the proof.
\end{proof}

\begin{thm}[Main Theorem]
Fix $n,k \geq 20$. Let $M$ be our algorithm  from Figure \ref{fig:scaledgauss} with $T = n^2/k$ and $t=n\sqrt{2\ln(2k)/k}$ that answers $k$ statistical queries given $n$ samples.

Let $\cP$ be a distribution on $\X$ and let $A$ be an interactive algorithm that asks $k$ statistical queries. Then $$\ex{S \sim \cP^n \atop (\psi,v) \sim \interact{A}{M}(S)}{\max_{j \in [k]} \frac{|v_j-\cP[\psi_j]|}{\mymax{\tau\cdot\sdv(\psi_j(\cP))}{\tau^2}}} \leq 4,$$
where $\tau=\sqrt{\frac{\sqrt{2k\ln(2k)}}{n}}$.
\end{thm}
\begin{proof}
Set $\varepsilon = \tau^2 = \frac{kt}{n^2} = \frac{t}{T} = \frac{\sqrt{2k \ln(2k)}}{n}$. Since $n \geq 20$ and $T \leq \min\{t^2,nt/10\}$, $\interact{A}{M}$ is $\varepsilon$-\KLAS{} by Theorem \ref{thm:isklas}.
Let $S \sim \cP^n$ and $(\psi,v) \sim \interact{A}{M}(S)$. Define $$j_* = \underset{j \in [k]}{\mathrm{argmax}} \frac{|v_j-\cP[\psi_j]|}{\mymax{\sdv(\psi_j(\cP))}{\tau}}$$ and $$(\psi_*,v_*) = \left\{\begin{array}{cl} (\psi_{j_*},v_{j_*}) & \text{ if } v_{j_*} \geq \cP[\psi_{j_*}] \\ (1-\psi_{j_*},1-v_{j_*}) & \text{ if } v_{j_*} < \cP[\psi_{j_*}]\end{array}\right\}.$$
Thus \begin{equation}\ex{S \sim \cP^n \atop (\psi,v) \sim \interact{A}{M}(S)}{\max_{j \in [k]} \frac{|v_j-\cP[\psi_j]|}{\mymax{\tau \cdot \sdv(\psi_j(\cP))}{\tau^2}}} = \frac{1}{\tau}\ex{}{\frac{v_*-\cP[\psi_*]}{\mymax{\sdv(\psi_*(\cP))}{\tau}}}.\label{eqn:star}\end{equation}

Let $\sigma_*=\sigma_{j_*} = \sqrt{\frac{1}{n} \sum_{i=1}^n (\psi_{j_*}(S_i) - S[\psi_{j_*}])^2}$ be the empirical standard deviation corresponding to the query $\psi_*$ and the sample $S$.
By Lemma \ref{lem:generrtrue}, \begin{equation}\ex{}{\frac{S[\psi_*]-\cP[\psi_*]}{\mymax{\sdv(\psi_*(\cP))}{\tau}}} \leq 2\tau\label{eqn:generr-final}\end{equation} and \begin{equation}\ex{}{\left(\frac{\sigma_*}{\mymax{\sdv(\psi_*(\cP))}{\tau}}\right)^2} \leq 3.\label{eqn:emperr-final}\end{equation}

Let $\xi_1, \ldots, \xi_k$ be the independent standard Gaussians sampled by $M$ (Figure \ref{fig:scaledgauss}). Let $\xi_* = \xi_{j_*}$ if $v_{j_*} \geq \cP[\psi_{j_*}]$ and $\xi_* = -\xi_{j_*}$ if $v_{j_*} < \cP[\psi_{j_*}]$.
By the definition of our algorithm, $$v_* = S[\psi_*] + \xi_* \cdot \sqrt{\mymax{\sigma_*^2/t}{1/T}}= S[\psi_*] + \frac{1}{\sqrt{t}} \cdot \xi_* \cdot \mymax{\sigma_*}{\tau}.$$ Thus, by Cauchy-Schwartz,
\begin{align*}
\ex{}{\frac{v_*-\cP[\psi_*]}{\mymax{\sdv(\psi_*(\cP))}{\tau}}}
=& \ex{}{\frac{S[\psi_*]-\cP[\psi_*]}{\mymax{\sdv(\psi_*(\cP))}{\tau}}} + \frac{1}{\sqrt{t}}\cdot \ex{}{\frac{\xi_* \cdot \mymax{\sigma_*}{\tau}}{\mymax{\sdv(\psi_*(\cP))}{\tau}}} \\
\leq& \ex{}{\frac{S[\psi_*]-\cP[\psi_*]}{\mymax{\sdv(\psi_*(\cP))}{\tau}}} + \frac{1}{\sqrt{t}}\cdot\sqrt{\ex{}{\xi_*^2} \cdot \ex{}{\left(\frac{ \mymax{\sigma_*}{\tau}}{\mymax{\sdv(\psi_*(\cP))}{\tau}}\right)^2}} \\
\left(\text{by \eqref{eqn:generr-final}}\right)~~~\leq& 2\tau + \frac{1}{\sqrt{t}}\cdot\sqrt{\ex{}{\xi_*^2} \cdot \ex{}{\left(\frac{ \sigma_*}{\mymax{\sdv(\psi_*(\cP))}{\tau}}\right)^2 + \left(\frac{\tau}{\mymax{\sdv(\psi_*(\cP))}{\tau}}\right)^2}} \\
\left(\text{by \eqref{eqn:emperr-final}}\right)~~~\leq& 2\tau + \frac{1}{\sqrt{t}}\cdot\sqrt{\ex{}{\xi_*^2} \cdot (3 + 1)} \\
\left(\text{by Lemma \ref{lem:gaussmax}}\right)~~~\leq& 2\tau + \frac{1}{\sqrt{t}}\cdot\sqrt{2 \ln(2k) \cdot 4}\\
=& 4\tau = 4 \sqrt{\frac{\sqrt{2k\ln(2k)}}{n}}.
\end{align*}
Combining with \eqref{eqn:star} completes the proof.
\end{proof}

\subsection*{Acknowledgements} We thank Adam Smith for his suggestion to analyze the generalization of ALKL stable algorithms via mutual information. This insight greatly simplified our initial analysis and allowed us to derive additional corollaries presented in Section \ref{sec:events}. We also thank Nati Srebro for pointing out the connection between our results and the PAC-Bayes generalization bounds.


\printbibliography

\end{document}